\documentclass[letterpaper, 10 pt, conference]{ieeeconf}  
\IEEEoverridecommandlockouts                              % This command is only needed if 
\usepackage{amsmath,amsfonts,amsthm,amssymb}  % math
\usepackage[font=footnotesize]{caption}
\usepackage[caption=false, font=footnotesize]{subfig}
\usepackage{algorithmic}      % algorithms
\usepackage{graphicx}         % figures (PNG/JPG/PDF)
\usepackage{svg}              % SVG graphics
\usepackage{xcolor}           % coloured text / links
\let\labelindent\relax
\usepackage{enumitem}   % in the preamble
\usepackage{tikz}       % for drawing figures
\usepackage{soul} % required for \hl command
\usepackage{wrapfig}
\usepackage{multirow}
\usepackage{cite}             % bibliography handling
\usepackage{multicol}
\usepackage[bookmarks=true]{hyperref}

\allowdisplaybreaks

\usepackage{multirow}
\usepackage{xspace}
\usepackage{cleveref}
\usepackage{bm}

\usepackage{etoolbox}
\newtoggle{arxiv}
\toggletrue{arxiv}

\newtheorem{theorem}{Theorem}
\newtheorem{claim}{Claim}

\graphicspath{{Model_Control_Figs/}{Simulation_Results/}{Simulation_Results/Allocation_results/}}

\usepackage{graphicx}
\usepackage{subcaption}
\usepackage{float}

\newcommand{\ignore}[1]{}

\def\T{\mathcal{T}}

\def\dR{\mathbb{R}}

\def\epsilon{\varepsilon}

\newif\ifshowcomments
\showcommentstrue          % <-- set to \showcommentsfalse to hide all comments

\ifshowcomments
  \newcommand{\kiril}[1]{\textcolor{red}{(\textbf{Kiril:} #1)}}
  \newcommand{\adir}[1]{\textcolor{orange}{(\textbf{Adir:} #1)}}

  \newcommand{\snir}[1]{\textcolor{blue}{(\textbf{Snir:} #1)}}
  \newcommand{\todo}[1]{\textcolor{cyan}{(\textbf{TODO:} #1)}}
\else
  \newcommand{\kiril}[1]{\ignorespaces}
  \newcommand{\adir}[1]{\ignorespaces}
  \newcommand{\snir}[1]{\ignorespaces}
  \newcommand{\todo}[1]{\ignorespaces}
\fi

\newcommand\algname[1]{\textsf{#1}\xspace}

\def\pid{\algname{PID}}
\def\fbl{\algname{FBL}}

\def\niceparagraph#1{\vspace{5pt} \noindent \textbf{#1}}

\theoremstyle{definition}
\newtheorem{problem}{Problem}

\newif\ifincludeappendix
\includeappendixtrue  % Default is false; change to true when needed

\begin{document}

% --------------------------------------------------
%  Title & author block
% --------------------------------------------------
\title{%Multi-Robot Path Planning Optimization for Oil Spill Confinement and Cleanup at Sea
Routing and Control for Marine Oil-Spill Cleanup \\ with a Boom-Towing Vessel Fleet
}
\iftoggle{arxiv}{
\author{Snir Carmeli, Adir Morgan and Kiril Solovey% <-this % stops a space
\thanks{The authors are with the Viterbi Faculty of Electrical and Computer Engineering, Technion--Israel Institute of Technology, Haifa, Israel. 
\{\tt snircarmeli, samorgan\}@campus.technion.ac.il, kirilsol@technion.ac.il}%
}
}
{
\author{Author details removed for double-blind review
% <-this % stops a space
}
}

% \author{
%     \IEEEauthorblockN{Snir Carmeli, Adir Morgan and Kiril Solovey}
%     \authorblockA{Viterbi Faculty of Electrical and Computer Engineering\\
% Technion--Israel Institute of Technology, Haifa, Israel\\
% \{snircarmeli, samorgan\}@campus.technion.ac.il, kirilsol@technion.ac.il}
% }

% Render title
\maketitle
\begin{abstract}
Marine oil spills damage ecosystems, contaminate coastlines, and disrupt food webs, while imposing substantial economic losses on fisheries and coastal communities. Prior work has demonstrated the feasibility of containing and cleaning individual spills using a duo of autonomous surface vehicles (ASVs) equipped with a towed boom and skimmers. However, existing algorithmic approaches primarily address isolated slicks and individual ASV duos, lacking scalable methods for coordinating large robotic fleets across multiple spills representative of realistic oil-spill incidents.
In this work, we propose an integrated multi-robot framework for coordinated oil-spill confinement and cleanup using autonomous ASV duos. We formulate multi-spill response as a risk-weighted minimum-latency problem, where spill-specific risk factors and service times jointly determine cumulative environmental damage. To solve this problem, we develop a hybrid optimization approach combining mixed-integer linear programming, and a tailored warm-start heuristic, enabling near-optimal routing plans for scenarios with tens of spills within minutes on commodity hardware.
For physical execution, we design and analyze two tracking controllers for boom-towing ASV duos: a feedback-linearization controller with proven asymptotic stability, and a baseline PID controller. Simulation results under coupled vessel–boom dynamics demonstrate accurate path tracking for both controllers. 
Together, these components provide a scalable, holistic framework for rapid, risk-aware multi-robot response to large-scale oil spill disasters.

 %and solve it using an effective mixed-integer linear program . We generate obstacle-aware, kinodynamically feasible paths for boom-towing pairs and propose two different control algorithms: a feedback-linearization-based controller with proven asymptotic stability with ease of tuning, and a baseline PID controller for ease of implementation. Simulation results demonstrate a comparison between different routing algorithms, and the path-following metrics of the two controllers.
% Marine oil spills pose a persistent, catastrophic ecological threat
% with severe economic repercussions.  mechanical recovery using containment booms and skimmers remains a primary response strategy when environmental and regulatory conditions permit. Autonomous surface vehicles (ASVs) offer the potential for persistent, coordinated cleanup operations, and prior work has demonstrated automated boom towing and single-spill containment. But existing approaches largely focus on isolated slicks, simplify routing and dynamics, and lack guarantees of completeness for control path-following algorithms.
%
\end{abstract}

\iftoggle{arxiv}
{}
{}

\section{Introduction}
The environmental and economic repercussions of marine oil spills are extensive~\cite{KINGSTON200253, smith2011analysis, saadoun2015impact}, highlighting the urgent need for comprehensive research and mitigation strategies. 
%The catastrophic impact on marine ecosystems, coastal communities, and economies demands rapid, efficient methods for containing and cleaning up spills.  
%Commonly deployed techniques include surface skimmers, boom towing, controlled in‑situ burning, sorbents, and chemical dispersants.
Despite decades of safety improvements, spills continue to occur annually.  
In 2024 alone, there were six large tanker spills, each exceeding 700\,t (metric tons), and four medium tanker spills (7–700\,t). The current decade averages 7.4 incidents per year above 7\,t~\cite{ITOPF_Stats_2024}. %—over 90\% fewer than in the 1970s, yet non‑negligible in consequence
Beyond tankers, chronic and operational releases from platforms, pipelines, and coastal activities remain significant contributors to marine oil inputs~\cite{NRC2003_OilInTheSeaIII}. 

Oil can persist in sensitive habitats for years, disrupting marine food webs and degrading ecosystem function \cite{KINGSTON200253,saadoun2015impact}. Major spills also impose substantial response costs and prolonged economic losses to fisheries and tourism. For example, the 2010 Deepwater Horizon spill resulted in over \$65 billion in total costs, including \$14 billion for cleanup operations, \$20 billion in penalties to the BP oil company, and billions more in lost revenue for Gulf Coast fisheries and tourism industries \cite{smith2011analysis}. 
These harms motivate rapid, near‑source confinement and efficient, sustained recovery, especially in nearshore settings where even modest spills can have major impacts.

In response to oil spillage, mechanical recovery with booms and skimmers is often preferred~\cite{NRC2003_OilInTheSeaIII, Fingas2016OSST}. %, complemented by dispersants or in‑situ burning where environmental conditions and regulations permit 
A boom confines the oil that floats on the water surface to a small area and increases its layer thickness. This enables the use of skimmers, which are floating pumps deployed inside the boom‑restricted area, to collect oil from the water surface~\cite{pereda2011towards}. 
% However, effectiveness can be sharply reduced by sea state, encounter rates, towing geometry, and access/logistics constraints; decanting and effluent treatment requirements further complicate continuous operations near shore \cite{NRC2003_OilInTheSeaIII, BSEE_2018_Demulsification, BSEE_2005_Decanting, EPA_2000_CWT_DAF, IMO_MARPOL_AnnexI_2022, IMO_MEPC_107_49, CFR_33_151_10}. 
However, a manual implementation of such a solution struggles to scale across multiple, spatially separated slicks and to maintain persistent operations under evolving conditions (winds, currents, and navigational hazards).
Overall, three gaps repeatedly surface in post‑incident reviews~\cite{NRC2003_OilInTheSeaIII, ITOPF_Stats_2024}: (i) delayed arrival at the highest‑risk slicks, (ii) inefficient utilization of assets across many small and medium‑size patches as conditions evolve, and (iii) conservative near-shore operations due to discharge restrictions and safety requirements. These gaps motivate further research into technological means for treating oil spills and introduce algorithmic challenges for autonomous robotic solutions.   

%emphasize both  first two gaps describe, in algorithmic terms, a coupled task‑allocation and kinodynamic multi‑robot path‑planning problems.
 
% While dynamic spill evolution and shore‑proximity constraints represent important extensions for future work, this research focuses on the core allocation and routing challenges for static spill scenarios in open water.

\begin{figure}
    \centering
    \includegraphics[width=0.4\textwidth]{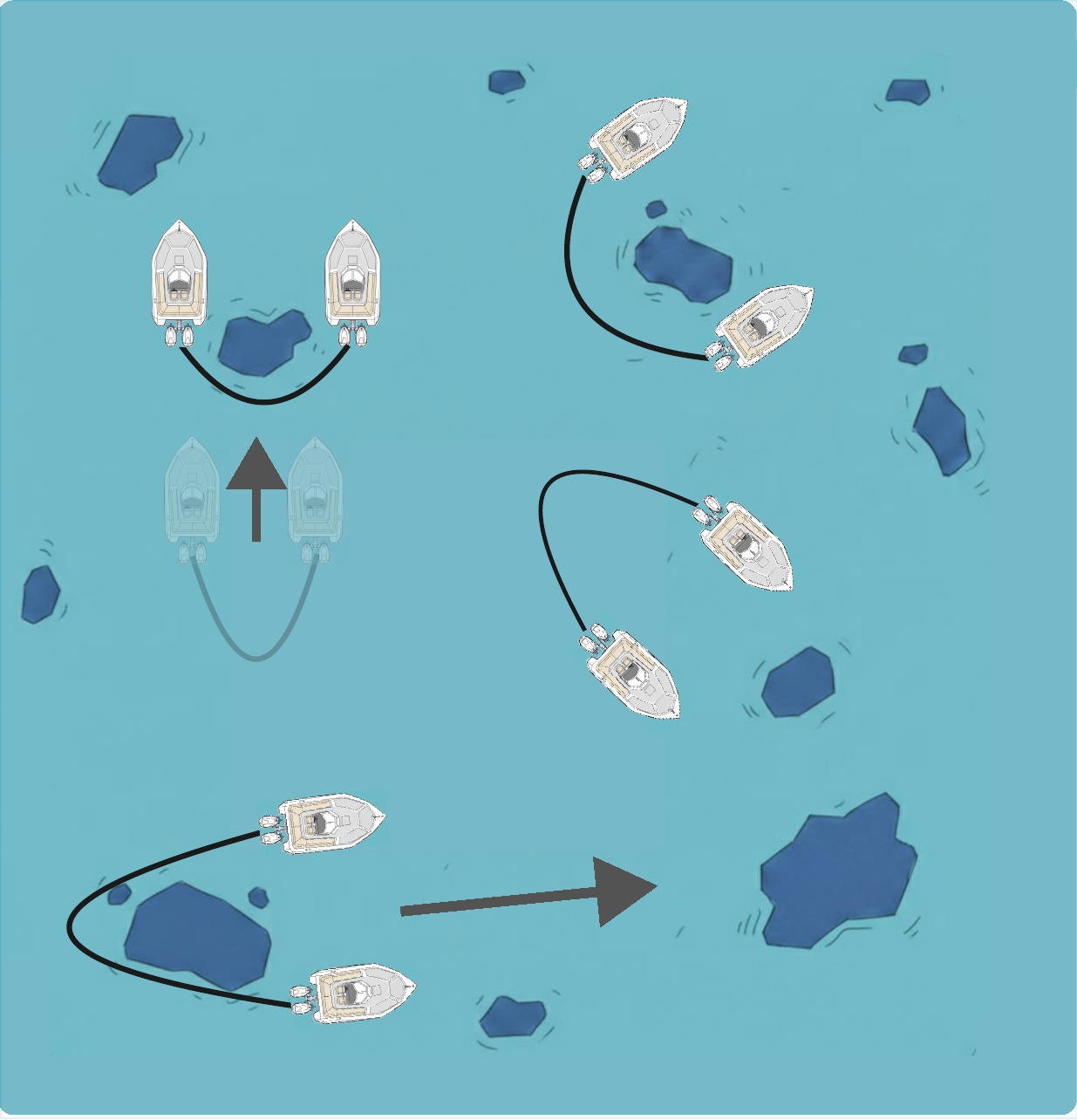}
   \caption{Visualization of coordinated oil-spill confinement and cleanup using autonomous ASV duos equipped with towed booms and skimmers, which we tackle in this work.}
    \label{fig:spills_scenario}
    \vspace{-20pt}
\end{figure} 

\niceparagraph{Contribution.}
We develop an effective integrated algorithmic framework for autonomous oil-spill containment and cleanup using a coordinated team of boom-towing ASV duos. Our key contribution lies in careful modeling of the problem (Sec.~\ref{sec:model}), which gives rise to two computationally-tractable subproblems: (1) routing a team of ASV duos between the oil spills while minimizing the damage incurred by the spills (Figure \ref{fig:spills_scenario}), and (2) a closed-loop path tracking approach  for executing the routing solution by each ASV duo (Figure \ref{fig:setpoints-trajectory}). 
% care
% %
% Our key insight lies in modeling the problem into two interconnected  as damage-aware planning and tracking (Sec.~\ref{sec:model}). , allowing us to narrow computational challenges into two separate problems. %---damage-optimal routing of ASV-duos, and path tracking to  developing a controller for such duos.
%

We address the first problem (Sec.~\ref{sec:routing}) by reducing it to a multi-agent variant of the \emph{traveling repairman problem}~\cite{Blum1994MLP} and by developing a solver that combines  mixed-integer linear programming, and heuristic techniques, capable of tackling realistic problem instances within minutes (Sec.~\ref{sec:exper-eval}).

For the second problem, we develop (Sec.\iftoggle{arxiv}{~\ref{sec:PathTracking}}{~\ref{sec:feedback}}) an approach to execute the routing trajectory for a given ASV duo as a path-tracking problem, without violating boom-induced constraints that connect the two vessels. We then explore two closed-loop control approaches to accomplish path tracking, while accounting for the complex dynamics of the boom-towing duo. Specifically, we consider a baseline PID approach and a feedback-linearization controller with proven asymptotic stability, both of which are shown, through extensive simulations (Sec.~\ref{sec:exper-eval}), to achieve accurate path tracking. 

Together, these components provide a scalable, holistic framework for rapid, risk-aware multi-robot response to large-scale oil spill disasters.

\begin{figure}[t]
    \centering
    \vspace{5pt}
\includegraphics[width=0.55\columnwidth, angle=90]{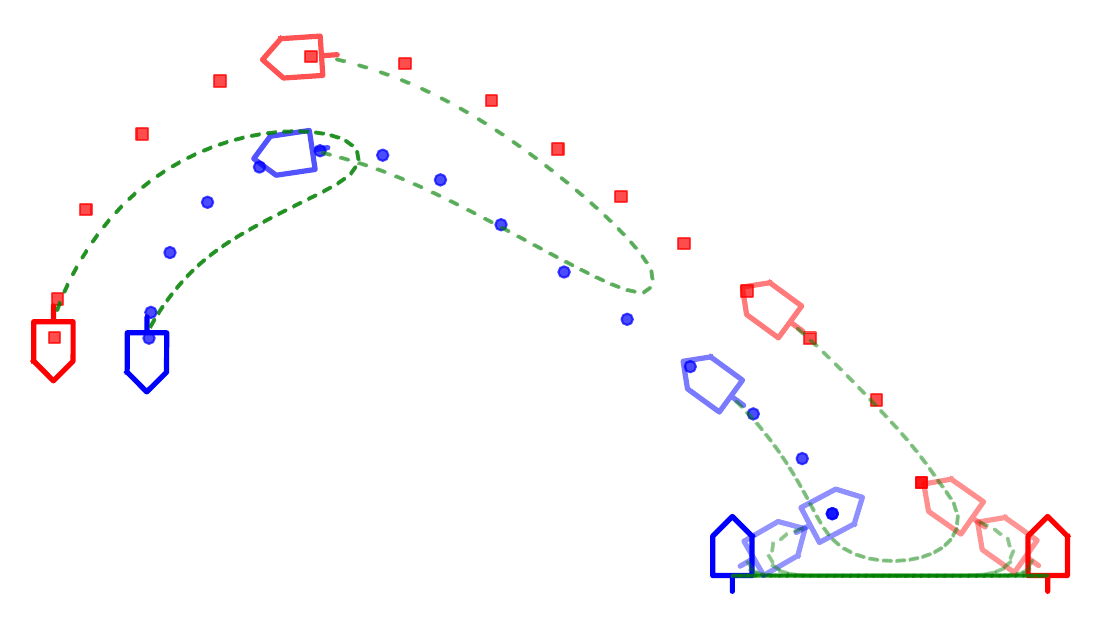}
    \caption{Illustration of the path tracking problem for a boom-towing ASV duo, along with feedback-linearization control behavior, presented in Sec.~\ref{\iftoggle{arxiv}{sec:PathTracking}{sec:feedback}}.  A reference trajectory generated by the routing approach is converted into a sequence of tracked setpoints (blue dots and red squares) for each individual ASV, while satisfying boom constraints.}
    \label{fig:setpoints-trajectory}
    \vspace{-15pt}
\end{figure}

\section{Related Work}
\iftoggle{arxiv}{
A substantial body of literature addresses sensing, monitoring, mechanical confinement, and cleanup for marine oil spills. 

  For \emph{sensing}, a work by Leifer~\cite{leifer2012state} reviews satellite and airborne remote sensing together with in-situ methods.
Additional work by Kumar et al. \cite{kumar2020efficient} develops an efficient path-planning strategy for AUV-based oil-spill detection in coastal waters, using a hybrid evolutionary optimization algorithm to minimize search distance, time, and energy consumption while ensuring high spill coverage.
More recently, Dong et al.~\cite{dong2025review} provide a comprehensive overview of data-driven analysis of remote-sensing data for marine oil spill detection, classification, and thickness estimation. 
Kuzmenko et al.~\cite{kuzmenko2025autonomousoilspillresponse} propose a learning-based approach for real-time oil-spill trajectory prediction, along with a multi-agent tracking. 
%While this architecture demonstrates an integrated prediction–response loop, its multi-agent component is tailored to coordinating several vessels around a single evolving spill via heuristic path assignment and behavior-based control, whereas our framework targets a different scenario: allocating and routing boom-towing teams across multiple concurrent spills with explicit vessel–boom dynamics and provable completeness guarantees.
These sensing capabilities provide the foundation for spill detection and characterization.
In this work, we assume that spill locations, sizes, and characteristics are known through such monitoring systems, allowing us to focus on the subsequent response optimization challenge.
}{
  In this work, we assume that spill locations, sizes, and characteristics are given through existing  sensing or monitoring systems~\cite{leifer2012state,dong2025review,kuzmenko2025autonomousoilspillresponse}, allowing us to focus on the subsequent response optimization challenge.
}

For the task of containment of an oil spill, two-ship boom towing has been modeled and experimentally validated~\cite{giron2015preparing}. 
Controller design has been considered in~\cite{pereda2011towards}, although the specific implementation of the method on boom-vessel model and its use for path tracking have not been fully discussed. 
% the latter work, the authors adopt a null-space–based (NSB) behavioral controller with a clear mathematical formulation, but its adaptation to the specific boom–vessel model is only briefly discussed \kiril{Can you be more specific? Are there any critical details missing there?} \snir{fixed},\color{blue} as the modeling of \color{black} the boom tension forces \kiril{why is it difficult to measure those forces online?}\snir{It is not. In the model derivation (for simulation purposes) they didn't specify how those forces work.} is not fully documented. \kiril{It doesn't provide convergence guarantees, either, right?} \snir{Their control approach could yield convergence guarantees, but it needs to be accompanied by the desired path. We, on the other hand, explain how we combine the controller with the path-tracking, they don't. On a more intuitive way to explain it: if the road is hard to track, i.e. hard turns at high speeds, the math of the convergence proof needs to account for it.}
Another work~\cite{Arr2010} developed a controller for a caging task with an ASV duo, albeit without explicitly considering the towed boom in either theory or experiments. %In contrast, we consider the task of controller design in view of how it integrates within the broader mission goals, while at the same time accounting for realistic modeling assumptions and providing theoretical guarantees, whenever possible. 
In contrast, we frame controller design in the context of overarching mission objectives, incorporate realistic modeling assumptions, and offer theoretical guarantees whenever possible.

%consider the task of controller design in the context of the broad mission goals,  the broader  controllers with explicit modeling 
% connected by a flexible link has been demonstrated with disturbance-aware guidance and moment compensation, including pool experiments . 
% However, this line of work does not provide the formal stability guarantees that our model-based framework establishes.
%Another work employs reinforcement learning (RL) for inter-vehicle docking, enabling a team of ASVs to latch together and automatically preserve the desired boom shape, spacing, and overlap around the slick as they are pushed around by tension and drag forces \cite{kim2012toward}. 
%
%The works mentioned so far target a single slick with a two-vessel formation and largely ignore obstacles and kinodynamic constraints. 
%Moreover, existing control approaches for vessel-boom systems suffer from key limitations: RL-based methods lack general stability guarantees and may fail when the system operates outside the conditions seen during training \cite{Busoni2018, Ye2021}. For \emph{cleanup}, Elmakis and Degani study multi-destination RL within a port with a single USV \cite{Elmakis}. 
%
%\adir{TODO - SHORT lit. review}

At a higher decision-making level, 
a recent approach leverages classical planning and reinforcement learning for multi-spill cleanup using a single ASV~\cite{Elmakis}. 
Several works formulate oil-spill containment and response as optimization problems over emergency resources, depots, and transportation networks~\cite{xu2025boomScheduling,zhang2021ejorOilScheduling,zhanglu2024timeVaryingLRP}.  %typically using time-varying travel times and demand profiles rather than continuous-time vehicle dynamics.
% Xu et al.~\cite{xu2025boomScheduling} develop a multi-objective scheduling model for deploying containment booms under dynamically changing spill conditions, explicitly trading off response time with quantified economic and ecological losses.
% Zhang et al.~\cite{zhang2021ejorOilScheduling} study optimal scheduling of maritime oil-spill emergency resources while accounting for time-varying demand and disruptions in the transportation network, and Zhang and Lu~\cite{zhanglu2024timeVaryingLRP} propose a time-varying multi-resource location-routing formulation for oil-spill emergency logistics.
 % are closely related to our \emph{allocation and routing} component, but their problem setting and algorithmic pipeline impose several limitations for the autonomous ASV routing setting considered here.
However, the latter contributions rely on logistics abstraction (regions, depots, resource types, and network arcs), which cannot immediately translate to actionable ASV plans, and employ population-based metaheuristics with limited predictability in runtime and solution quality. In this context, we develop an efficient optimization framework for the multi-ASV oil-spill response problem while capturing intrinsic problem attributes (e.g., spill geometries and ASV capabilities).

\section{Damage-Aware Routing and Tracking model}\label{sec:model}
In this section, we formalize the oil-spill response problem. We begin by describing a set of operational assumptions that motivate a graph-based abstraction. This abstraction leads to two coupled computational problems: damage-aware routing and trajectory tracking by boom-towing ASV duos.

\subsection{System and Response Model}
We consider a multi-ASV oil-spill response system operating over a planar maritime workspace 
$\mathcal{W} \subset \mathbb{R}^2$, representing a navigable water surface. 
The workspace excludes static obstacles and restricted areas such as shorelines and offshore structures.
Within $\mathcal{W}$, a set of $p$ oil spills $\{S_1,\ldots,S_p\}$ is detected by the system operator. 
Each spill $S_i$ is modeled as a time-invariant region of $\mathcal{W}$ with known volume $\mathcal{V}(S_i)$ and perimeter $C(S_i)$ . 
Spill growth, stochastic drift, and environmental uncertainties are not considered in this formulation. %Spill perimeters and volumes are assumed to be known a priori. 

To respond to an oil-spill scenario, a fleet of $k$ ASV duos $\{B_1,\ldots,B_k\}$ is dispatched from a common depot $d \in \mathcal{W}$ at time $t=0$. 
Each duo consists of two surface vessels towing a containment boom of fixed length $L$.
All ASV duos are assumed homogeneous in transit speed, containment capability, and cleaning rate.  
Upon reaching a spill, a duo performs a containment maneuver: the vessels pass the spill from opposing sides and corral it within the curvature limits of the boom. This is followed by oil skimming using onboard treatment systems.
The cleaning time for spill $S_i$ is modeled as
$
t_i^{\text{clean}} = \alpha^{\text{clean}} \cdot \mathcal{V}(S_i)$,
where $\alpha^{\text{clean}}$ is a known oil-removal rate constant.
We assume that each ASV is equipped with an onboard treatment system capable of separating oil from recovered mixtures and discharging compliant water \cite{BSEE_2018_Demulsification, BSEE_2005_Decanting, EPA_2000_CWT_DAF}. 
This capability enables sustained multi-spill operations without frequent returns to shore for offloading. %Accordingly, we allow routes to terminate after servicing the final assigned spill, without requiring return to the depot.

A route $\tau_j$ assigned to duo $B_j$ specifies both the sequence of spills it services and the corresponding feasible trajectory in $\mathcal{W}$. 
Each route must originate at the depot $d$. 
The total execution time of a route includes both transit times between spills and spill-specific cleaning times.
Given a route set $\T = \{\tau_1,\ldots,\tau_k\}$, we define the elimination time $t_{\T}(S_i)$ as the time at which spill $S_i$ is fully cleaned.

% Specifically, for containment, ASV duos tow booms to corral oil patches and prevent further spreading. In addition, each ASV is equipped with an onboard treatment system %using chemical flocculation combined with dissolved-air flotation (DAF) technology 
% \cite{BSEE_2018_Demulsification, BSEE_2005_Decanting, EPA_2000_CWT_DAF}, which discharges compliant water and retains only concentrated oil and sludge for disposal. 
% Critically, this eliminates the need for frequent shore returns for tank offloading, enabling sustained operations across multiple spill sites. 
% \adir{I think its diving to deep to topics we aren't addressing in our work. } \snir{It is important because if we don't mention it, then we can't clean a lot of oil without discharging it somewhere. we show that there is a mechniacl/chemical solution to fit our setting of the problem.} \adir{Yet we can state it as our contribution, and we can't discuss something here that doesn't appear in the rest of the paper. IMO it should be part of our modeling, together with the other marine-specific assumptions. (intro of section 3, maybe it should be a subsection of its own.) }
%

\subsection{Damage Accumulation as Latency Objective}
Each spill $S_i$ is associated with a risk weight $R_i > 0$, reflecting its environmental severity (e.g., volume or proximity to sensitive areas). 
We model environmental damage as accumulating linearly over time until elimination, as a first-order approximation of exposure-driven impact: the longer a spill remains untreated, the greater the cumulative environmental harm. 
While real spill dynamics may exhibit nonlinear spreading and ecological effects, a linear accumulation model captures the essential urgency of early intervention while preserving the tractability of the planning problem.
Thus, for a given route set $\T$, the total damage incurred by spill $S_i$ is $R_i \, t_{\T}(S_i)$.
The total cumulative damage is therefore
\begin{equation} \label{eq:objective}    
    \sum_{i=1}^p R_i \, t_{\T}(S_i). 
\end{equation}
Minimizing cumulative damage is thus equivalent to minimizing a weighted sum of spill completion times. 
% \adir{Touch the "early stages matters most" topic}

\subsection{Motion-Based Graph Representation}
\label{Motion_Based_GR}
To enable scalable task-level planning, we abstract the continuous routing problem (associated with optimizing Eq.~\eqref{eq:objective}) into a directed graph representation $G=(V,E)$. The vertex set represents the depot and spill locations, i.e. $V = \{d\} \cup \{S_1,\ldots,S_p\}$.
Edges $(i, j)\in E$ correspond to feasible ASV duo motion segments between spills. 
The cost of such motion $c_{ij}$ represents the time required for a duo to transit from $i$ to $j$ along a bounded-curvature, obstacle-avoiding path, and perform containment and cleaning at $j$ (if $j\neq d$). 
For spill-to-spill transitions, we assess motion time as: 
\[
c_{ij}
=
% \frac{C(S_i)}{v_{\textup{encircle}}}
% +
\frac{d(S_i,S_j)}{v_{\textup{transit}}}
+
\frac{C(S_j)}{v_{\textup{encircle}}}
+
t_j^{\text{clean}},
\]
where 
$d(S_i,S_j)$ is the length of the shortest curvature-constrained collision-free path between spill centers, $C(S_j)$ denotes the perimeter of spill $S_j$. The values  $v_{\textup{transit}}$ and $v_{\textup{encircle}}$ represent the transit and containment maneuver speeds, respectively. 
%
% \begin{itemize}
%     \item $d(S_i,S_j)$ is the length of the shortest curvature-constrained collision-free path between spill centers,
%     \item $C(S_j)$ denotes the perimeter of spill $S_j$,
%     \item $v_{\textup{transit}}$ is the constant transit speed,
%     \item $v_{\textup{encircle}}$ is the containment maneuver speed.
% \end{itemize}

The graph abstraction assumes that edge costs $c_{ij}$ correspond to dynamically feasible motion segments between spills, consistent with the boom-towing ASV duo dynamics and controller capabilities (Sec.~\iftoggle{arxiv}{
\ref{sec:BoatDynamics}, ~\ref{sec:Controller-Definition}}{\ref{sec:feedback}}). In principle, curvature-constrained motion planning could be employed to ensure strict kinematic feasibility of the duo.
For the purposes of evaluation, however, we adopt a geometric approximation: the workspace $W$ is discretized into an occupancy grid, and obstacle-avoiding shortest paths between representative spill locations (e.g., centroids) are computed using $A^*$ search. The resulting path length is used to estimate the transit component of $c_{ij}$, while the velocity parameters $v_{\textup{transit}}$ and $v_{\textup{encircle}}$ are calibrated based on controller performance (Sec.~\ref{sec:control-results}).

\subsection{Damage-Minimizing Routing Problem} \label{sec:routing-problem-definition}
Using the graph abstraction introduced above, the oil-spill response task reduces to selecting $k$ routes in $G$ that minimize the cumulative weighted completion time of all spills. Each route $\tau_j\in \T$ starts at the depot and services a subset of spills, where the completion time $t_{\T}(S_i)$ of spill $S_i$ is determined by its arrival and service time along the assigned route.

\begin{problem}[Damage-Minimizing Routing]\label{prob:planning}
Given $G=(V,E,c)$, number of ASV duos $k$, and spill risk-factors $\{R_i\}$, find a route set $\T=\{\tau_1,\ldots,\tau_k\}$ that minimizes %the objective in 
Eq.~\eqref{eq:objective}.
\end{problem}

Algorithmically, Problem~\ref{prob:planning} corresponds to a weighted multi-agent variant of the \emph{traveling repairman problem} (TRP)~\cite{muritiba2021branch,Blum1994MLP}. %also known as the Minimum Latency Problem (MLP)~\cite{Blum1994MLP}.
% Even the single-agent TRP is NP-hard~, and the weighted and multi-agent generalizations considered here remain NP-hard. 
This structural connection allows us to leverage algorithmic tools developed for TRP, providing efficient and scalable solvers for the presented problem (Sec.~\ref{sec:routing}).

\subsection{Trajectory Tracking Control Problem}
The planning layer produces, for each duo, a reference trajectory $\tau_j$ in the continuous workspace. 
To ensure the physical execution of the reference, we formulate a tracking control problem.
\begin{problem}[Trajectory Tracking for Boom-Towing Duo] \label{prob:control}
Given a trajectory $\tau_j$, design a path-tracking controller for the full boom-towing ASV duo system, for executing the reference in a stable and error-minimizing manner.
\end{problem}

\iftoggle{arxiv}
{
We consider the dynamics of a duo system in Sec. \ref{sec:dynamics}, and develop a solution addressing Problem~\ref{prob:control} in Sec \ref{sec:PathTracking}.
}
{
In 
Sec.~\ref{sec:feedback} we consider the dynamics of a duo system and develop a solution addressing Problem~\ref{prob:control}.}

\section{Solving Damage-Minimizing Routing} \label{sec:routing}
Leveraging the modeling framework defined in Sec.~\ref{sec:routing-problem-definition}, we now address the damage-minimizing routing (DMR) problem (Problem \ref{prob:planning}) algorithmically. Since DMR is a weighted multi-agent variant of TRP, and therefore NP-hard \cite{Blum1994MLP}, we adopt a mixed-integer linear programming (MILP) formulation solved via branch-and-bound (BnB) as the primary method for multi-agent coordination.
While the MILP provides optimality guarantees, scalability considerations motivate accelerating its convergence. We therefore complement it with a structured warm-start heuristic that supplies high-quality initial solutions to the BnB solver. This heuristic is augmented with an exact single-agent\footnote{We use the term "agent" to refer to a boom-towing ASV duo.} dynamic programming (DP) refinement step, ensuring locally optimal routes under fixed spill assignments. 
As demonstrated in Sec.~\ref{sec:exper-eval}, the hybrid MILP and DP-augmented warm-start framework yields routing solutions that are frequently near-optimal, as certified by tight lower bounds obtained during branch-and-bound, while remaining computationally efficient for large spill instances. We will provide a complete open-source implementation of the solver upon publication. 
% \footnote{The implementation will be made publicly available upon publication.}

\subsection{MILP Formulation for DMR}
To solve the multi-agent DMR problem, we formulate it as a mixed-integer linear program (MILP). Rather than employing a time-indexed or explicit completion-time formulation, which would introduce large time-expanded state spaces, we adopt an edge-centric representation of the weighted latency objective. This representation yields a tighter linear relaxation and improves scalability under branch-and-bound.

We adopt the MILP formulation of Muritiba et al.~\cite{muritiba2021branch} for the weighted $k$-TRP as the foundation of our solver. We further introduce structural modifications to reflect the spill-response setting. For brevity, we present only the essential components here, with full implementation details and solver configurations provided in our open-source repository.

% We adopt the wkTRP MILP formulation of Muritiba et al.~\cite{muritiba2021branch} as our core optimization framework, and adapt it to reflect the structural characteristics of spill response, including terminal service without depot return

% We build our solution upon the framework , originally developed for the weighted wk-TRP by exploiting its linear ordering structure~\cite{mendez2008new}. We further introduce structural modifications to reflect the spill-response setting. For brevity, we present only the essential components here, with full implementation details and solver configurations are provided in our open-source repository.

\niceparagraph{Objective.}
By adopting an edge-centric reformulation for weighted latency, the cumulative-damage objective \eqref{eq:objective} can be expressed as:
\begin{equation}
\min 
\sum_{v \in V} R_v \cdot \sum_{(i,j)\in E} c_{ij} f_{ij}^v ,
\label{eq:milp_obj}
\end{equation}
where the decision variable $f_{ij}^v \in \{0,1\}$, indicates whether edge $(i,j)$ is traversed by an agent on the route that eventually services spill $v$.
Consequently, the completion time of spill $v$ is $t_{\T}(v)=\sum_{(i,j)\in E} c_{ij} f_{ij}^v$. 
While this variable-per-spill approach induces a large formulation of $O(|V|^3)$ variables, it allows the objective to naturally account for multiple agents and the assignment of spills to agents.

\niceparagraph{Constraints.}
To ensure routing consistency across the multi-agent fleet, the formulation enforces the following constraints:
\begin{itemize}
    \item \textit{Fleet Initialization:} $\sum_{j=1}^n f_{0j}^j = k$ ensures exactly $k$ agents depart from the depot.
    \item \textit{Service Termination:} $\forall v, \sum_{i \neq v} f_{iv}^v \geq 1$ ensures that there is a route serving each spill   
    \item \textit{Route Continuity:} $\forall v, \sum_{i=1}^n f_{0i}^v = 1$ ensures that the route serving each spill $v$ originates at the depot $d=0$.
    \item \textit{Flow Consistency:} For each vertex $i \in V$ and each spill $j$ ($i \neq j$), we enforce $\sum_{k \neq i, j} f_{ij}^k \leq |V|\cdot f_{ij}^j$. This ensures that if an edge $(i, j)$ is used to reach a "later" spill $k$, it must also be marked as part of the path for the "immediate" spill $j$. 
    \item \textit{Subtour Elimination:} We utilize a set of constraints to forbid disconnected cycles, ensuring all routes are linked to the depot, using additional ordering variables.
\end{itemize}

\niceparagraph{Adaptations to Spill Response.}
Compared to the wkTRP formulation in \cite{muritiba2021branch}, we introduce several structural modifications to reflect the problem setting: 
\begin{itemize} 
    \item As environmental damage ceases upon spill treatment, and spills may be treated sequentially without depot return, we set no requirement for agents to depart from the final serviced node or return to the depot.
    \item All spill characteristics, such as treatment time, encircling etc., are integrated into motion graph edge weights.
\end{itemize}

\subsection{Warm-start Heuristic}\label{sec:method-heuristic}

Solving the MILP formulation exactly via branch-and-bound (BnB) can become computationally demanding as the number of spills grows. To improve scalability, we design a dedicated heuristic that produces high-quality feasible solutions at low computational cost. These solutions are subsequently used as warm-start incumbents for the BnB solver, and as we show in Sec.~\ref{sec:exper-eval}, significantly reduce the explored search tree and accelerate convergence. 
The heuristic consists of three stages:

\niceparagraph{(H1) Greedy Spill Assignment.} 
The first stage constructs an initial assignment of spills to ASV duos using a priority-driven greedy policy. To balance urgency and travel cost, we adopt an importance-to-travel-time criterion. Each agent maintains its accumulated route time, and agents are stored in a min-priority queue keyed by this value. At each iteration, we extract the agent with the smallest accumulated time and, letting $v$ denote the last spill assigned to it, select the unassigned spill $u^*$ such that:
\[
u^* = \arg\max_{u \in \mathcal{U}} \frac{R_u}{c_{vu}},
\]
where $R_u$ is the spill risk weight and $c_{vu}$ is the estimated travel-and-service time. This ratio serves as a greedy proxy for minimizing weighted completion time while promoting balanced workload distribution.

\niceparagraph{(H2) DP-based Visiting order optimization.}\label{dp}
Given the assignment from H1, we optimize the visiting order for each agent independently using dynamic programming (DP) for single-agent TRP~\cite{cormen2009introduction}.

For a spill subset $S \subseteq V$ and terminal spill $j \in S$, let $D(S,j)$ denote the minimum cumulative damage of any tour that starts at the depot, visits exactly the spills in $S$, and terminates at $j$.
Define for any subset $S'\subseteq V$ the total weight
\[
W_{S}= \sum_{u \in {S}} R_u,
\]
and define the remaining weight function: $\bar W_{S} := W_{V} - W_{S}$.
The DP recursion is then
\[
D[S, j]
=
\min_{v \in S \setminus \{j\}}
\left\{
    D[S \setminus \{j\}, v]
    +
    c_{vj}\cdot\bar{W}_{S \setminus \{j\}}
\right\}.
\]

The term $c_{vj}\cdot\bar{W}_{S \setminus \{j\}}$ reflects the latency objective: traversing edge $(v,j)$ increases the completion time of all unvisited spills by $c_{vj}$, and thus contributes proportionally to their total remaining weight. 
This ``edge-centric'' formulation avoids the need for a time-expanded state space, enabling us to preserve the $O(2^{|S|}|S|)$ complexity characteristic of the subset-DP method~\cite{held1962dynamic}.

In practice, this solver is limited by commodity RAM capacity to $|S|\leq 25$. As we present in Sec.~\ref{sec:exper-eval}, %using optimized implementation, 
this solver provides optimal solutions for problems in this scale in a few minutes, yielding a viable solver option for small to medium scale instances. When DP is too computationally costly, we only run the greedy ordering (H1). %to preserve tractability. 

\niceparagraph{(H3) Iterated local search.}
Finally, we refine the solution via an iterated local search (ILS) \cite{lourencco2003iterated} procedure that explores spill reassignment between agents. At each iteration, two spills are selected uniformly at random, and their assigned agents are swapped. For each affected agent, the visiting order is recomputed using the DP procedure of H2, and the total objective value is evaluated.
If the reassignment strictly improves cumulative damage, the new solution is accepted; otherwise, it is rejected. After a fixed number of iterations, the best solution found is returned and used to warm-start the MILP solver.

\iftoggle{arxiv}{

\section{Dynamics of a boom-towing ASV duo}
\label{sec:dynamics}
In preparation for designing a path tracking approach for a boom-towing boat duo (Sec. ~\ref{sec:PathTracking}), we discuss the dynamics of this system. We first consider the dynamics of each boat separately, as well as the boom dynamics, before considering the coupled system and discussing our modeling limitations. 

\subsection{Single ASV dynamics}
\label{sec:BoatDynamics}
We first describe the dynamics of an ASV duo towing a boom. 
Building on simplified single-ASV models ~\cite{Fossen2011}, we capture boom effects as an external load applied at each vessel’s tow point. 
Fig.~\ref{fig:hull-frame} shows the body-fixed surge–sway frame attached to the vessel's center of mass (CoM); inertial axes point East ($\mathbf{e_1}$) and North ($\mathbf{e_2}$), where $u$ and $v$ are the surge and sway velocities in the body frame. 
Control inputs are the propulsion thrust $F$ and the steering (rudder) angle $\eta$. Neglecting heave, roll, and pitch ~\cite{Fossen2011}, the rigid–body equations of motion (EOM) in the horizontal plane are
\begin{subequations}
\label{eq:bare-hull}
\begin{align}
\dot{u} &= \frac{F\cos\eta \;-\;\mu_{l}(u)\,u}{m}
          \;+\; \omega\,v, \label{eq:bare-hull-u} \\[4pt]
\dot{v} &= \frac{-\,F\sin\eta \;-\;\mu_{t}(v)\,v}{m}
          \;-\; \omega\,u, \label{eq:bare-hull-v} \\[4pt]
\dot{\omega} &= \frac{r\,F\sin\eta \;-\;\mu_{\omega}(\omega)\,\omega}{I},
          \label{eq:bare-hull-w}
\end{align}
\end{subequations}
Here,  $\omega$ is the yaw rate, $m$ is the mass, $I$ is the planar moment of inertia about the CoM, \(r\) is the propeller offset from the CoM, and \(\mu_{l},\mu_{t},\mu_{\omega}\) are the nonlinear drag coefficients that depend on the corresponding speed components, respectively: $
    \mu_* (f(t)) = \kappa_* \cdot |f(t)|$ for some $\kappa_* > 0$. 
The cross-coupling terms \(\omega v\) and \(-\omega u\)  represent Coriolis effects in the body frame.
\begin{figure}[H]  % “H” pins the float right HERE (needs the float package)
  \centering
  \includegraphics[width=0.8\linewidth]{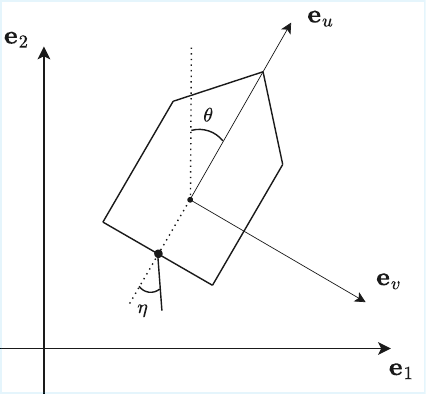}
  \caption{Body–fixed surge–sway frame \((u,v)\) attached to the vessel’s
           center of mass (CoM).}
  \label{fig:hull-frame}
\end{figure}

% Next, we consider the oil-containment boom, following  
% the derivation in \cite{pereda2011towards} \color{blue} with our contribution of inter-linking forces. \color{black} \kiril{I don't understand the addition. Are you saying that inter-linking forces were not considered in pereda2011towards? If that's the case, what was modeled in that paper? } \snir{They were considered, but their calculations were not published - They did not explain how they calculated the forces between links.} \kiril{So what is the point of mentioning their paper here if it gives no relevant information? }

\subsection{Boom dynamics}
\label{BoomDynamics}
The boom is modeled as a two-dimensional articulated chain
made up of \(n\geq 1\) identical, rigid links of length $L>0$.  Adjacent links are joined by
pin connections that are complemented with a spring-damper model of constant parameters. Figure~\ref{fig:link-frame} describes the tangential-normal frame ($\mathbf{t,n}$) fixed to each link’s CoM.
%This is similar to a previous model \cite{pereda2011towards}. %, accompanied by a model to estimate tension forces.

\begin{figure}[H]  % “H” pins the float right HERE (needs the float package)
  \centering
  \includegraphics[width=1.05\linewidth]{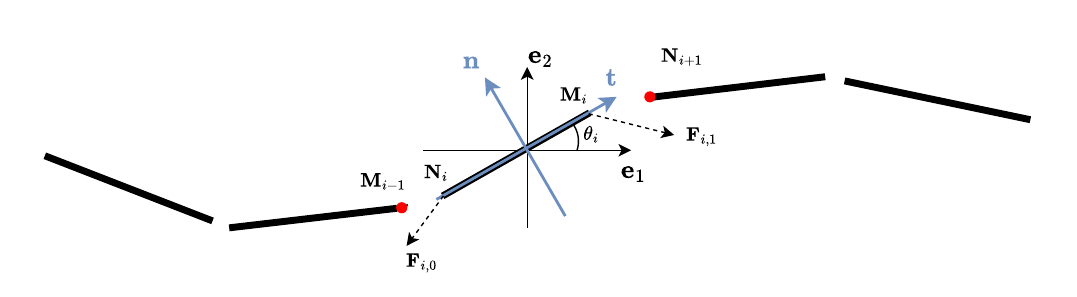}
  \caption{Visualization for the boom dynamics. Body–fixed surge–sway frame \((t,n)\) attached to the link’s
           center of mass (CoM); inertial axes point \textit{East} (\(e_1\)) and
           \textit{North} (\(e_2\)). The control inputs are the forces and torques applied by nearby linked bodies.}
  \label{fig:link-frame}
\end{figure}

The velocity of a link $1\leq i\leq n$ is given by
\[
\mathbf{v_i} = v_{t,i} \cdot \mathbf{t}_i + v_{n,i} \cdot \mathbf{n}_i, 
\]
where 
$\mathbf{t_i}$ is a unit vector of the $i$-th link in the direction of the link (t-tangent), and $\mathbf{n_i}$ is a unit vector, perpendicular ($n$-normal) to the $i$-th link such that $\mathbf{t_i} \times \mathbf{n_i}=\mathbf{e_3}$,  ( $\mathbf{e_3}$ is perpendicular to both $\mathbf{t_i}$ and $\mathbf{n_i}$). 
The angular velocity is:
\[
\mathbf{\omega_i}=\omega_i\cdot \mathbf{e_3}
\]
We model two types of forces that are applied to each link $i$. %The hydrodynamical drag on 3 axes and a tension force from each adjacent link, which is modeled as a spring-damper force.
The hydrodynamic drag force $\mathbf{F}_{i, \text{drag}}$ and torque $\mathbf{\tau}_{i,\text{drag}}$ on link $i$, which are induced by the environments, are given by 
\begin{align*}
    \mathbf{F}_{i, \text{drag}} &= - \mu_t^*(v_{t,i})v_{t,i} \mathbf{t}_i - \mu_l^*(v_{n,i})v_{n,i} \mathbf{n}_i, \\
    \mathbf{\tau}_{i,\text{drag}} &= -\mu_{\omega}^*(\omega_i)\omega_i, 
\end{align*}
where $v_{t,i}$ and $v_{n,i}$ are the linear velocities of link $i$ in the tangential ($\mathbf{t}_i$) and normal ($\mathbf{n}_i$) directions, respectively. The terms $\mu_t^*, \mu_l^*,$ and $\mu_{\omega}^*$ represent nonlinear friction coefficients of the same form like in Sec. ~\ref{sec:BoatDynamics}

The inter-linking forces are modeled as a spring-damper system to keep the links close together (due to the spring) and to prevent oscillations (due to the damper).
Specifically, 
\begin{align*}
    \mathbf{F}_{i,1} = & k \cdot (\mathbf{N}_{i+1} - \mathbf{M}_i)  \\ & + c \cdot [(\dot{\mathbf{N}}_{i+1} - \dot{\mathbf{M}}_i) \cdot \mathbf{e}_{\mathbf{N}_{i+1}, \mathbf{M}_i}] \cdot \mathbf{e}_{\mathbf{N}_{i+1}, \mathbf{M}_i} , \\
    \mathbf{F}_{i,0} = & k \cdot (\mathbf{M}_{i-1} - \mathbf{N}_i) \\ & + c \cdot [(\dot{\mathbf{M}}_{i-1} - \dot{\mathbf{N}}_i) \cdot \mathbf{e}_{\mathbf{M}_{i-1}, \mathbf{N}_i}] \cdot  \mathbf{e}_{\mathbf{M}_{i-1}, \mathbf{N}_i},
\end{align*}
where $\mathbf{F}_{i,1}$ is the force applied on link $i$ by the link $i+1$ or a vessel (if it is the last link) and $\mathbf{F}_{i,0}$ is the force applied on link $i$ by the link $i-1$ or another vessel (if it is the first link).
The points $\mathbf{M}_i,\mathbf{N}_i\in \dR^2$ are the right and left edges of link $i$, respectively.
The constants $c$ and $k$ denote the linear damping and spring coefficients, respectively.
Additionally, $\mathbf{e}_{a, b}\in \dR^2$ is a unit vector going from $b\in \dR^2$ to $a\in \dR^2$.

% These forces \kiril{which ones?} are modeled as such for the following reasons:
% \begin{itemize}
%     \item Drag forces are applied on the links from their environment (the water).
%     \item I
% \end{itemize}

% The above model of the boom dynamics allows us to simulate each link's movement, and thus the entire boom's movement. Using this model, we now apply Newton-Euler equations \cite{Euler1765Theoria} \kiril{replace with a more useful resource. E.g., textbook on mechanics} to obtain a differential equation and solve for each timestep the next state of the boom. Applying Newton–Euler equations to each link yields the equations 

The boom model above enables simulation of each link and, consequently, the full boom configuration.
Using this model, we formulate the link dynamics via the Newton-Euler rigid-body equations, yielding a set of differential equations which we integrate forward in time to obtain the boom state at each timestep. Applying said equations to each link yields the equations  %\kiril{the following equation has notations that haven't been defined. $a,m,I,\theta$...}
\begin{subequations}
 \label{eq:boom-forces}
\begin{align}
m\,\mathbf{a}_i
&= m_{\text{link}}\Big[(\dot{v}_{t,i}-v_{n,i}\dot{\theta}_{\text{i}})\,\mathbf{t}_i
    +(\dot{v}_{n,i}+v_{t,i}\dot{\theta}_{\text{i}})\,\mathbf{n}_i\Big]
\notag\\
&= \mathbf{F}_{i,0}+\mathbf{F}_{i,1}+\mathbf{F}_{i,\text{drag}}, 
\label{eq:boom-force}\\[4pt]
I\,\dot{\omega}_i
&= \boldsymbol{\tau}_{i,\text{drag}}
  +\frac{L}{2}\,(\mathbf{F}_{i,1}-\mathbf{F}_{i,0})\cdot \mathbf{n}_i,
\label{eq:boom-torque}
\end{align}
\end{subequations}
where
$\mathbf{a}_i,m,I$ and  $\theta_i$ are the acceleration, mass, moment of inertia, and orientation of link $i$, respectively.

% The links on the edges of the boom are affected by forces from an adjacent link and from the boat, rather than other links, which are affected by forces from two adjacent links.  \kiril{I don't understand this statement. The end links are affected the boats and the links that are immediately attached to them. The latter, in turn, are also affected by other links in the system, all of the ingredients in this system affect each other.} 
\color{black}

For an interior link, the external forces include the interaction forces transmitted through its two adjacent joints (from the neighboring links on either side). 
% \kiril{how are those statments reflected in the above equations?}\snir{This statement and the one following can still be described by the same equations. The points $N_{n+1}$ and $M_0$ are the two boats. I've added this now.}
For an end link, only one adjacent joint transmits link-link interaction forces; the remaining boundary load is instead provided by the vessel attachment (i.e., the tow-point force/tension applied by the boat).
This is described as the points $N_{n+1}$ and $M_0$ (Fig. \ref{fig:link-frame}), which are the sterns of the boats. 
Thus, the full boom-vessel system remains dynamically coupled because forces propagate through the chain and the resulting tension feeds back into the vessel dynamics.

% ------------------------------------------------------------

\subsection{Combined duo-boom dynamics}
\label{CombinedDynamics}
% To obtain a full model of the system, we combine the boat dynamics with the boom dynamics. 
% Con
% Link 1 is attached to the stern of Ship 1, and link $n$ to the stern of Ship 2. The boom tension vector $\mathbf{f}_l$ acting on the vessel is assumed measurable (load cell).
% Equations \eqref{eq:bare-hull-u}–\eqref{eq:bare-hull-w}  now include the boom tension force $ \mathbf{f}_l $: \kiril{It's not immediately clear how equation \eqref{eq:boom-forces} comes into play here.}
Next we consider the combined dynamics of the coupled ASVs used in the controller design, where we  treat the boom as exerting an external load at the tow point on each vessel.

At every timestep, the boom dynamics (Eq.~\eqref{eq:boom-forces}) provide the tension/reaction force at the attachment link; by Newton's third law, the same force acts on the vessel with opposite sign.
We denote this tow-point load by $\mathbf f_l$ (expressed in the vessel's body frame), and assume it is measurable in practice (e.g., using a load cell at the stern attachment).
Substituting this external load into the bare-hull Equations~\eqref{eq:bare-hull} yields:
\begin{subequations}
\label{eq:disturbance-hull}
\begin{align}
\dot{u} &= \frac{F\cos\eta \;-\;\mu_{l}(u)\,u \; + \; \mathbf{f}_l \cdot \mathbf{e_u}}{m}
          \;+\; \omega\,v, \label{eq:bare-hull-u_d} \\[4pt]
\dot{v} &= \frac{-\,F\sin\eta \;-\;\mu_{t}(v)\,v \; + \; \mathbf{f}_l \cdot \mathbf{e_v}}{m}
          \;-\; \omega\,u, \label{eq:bare-hull-v_d} \\[4pt]
\dot{\omega} &= \frac{r\,F\sin\eta \;-\;\mu_{\omega}(\omega)\,\omega \; - \; r \cdot \mathbf{f}_l \cdot \mathbf{e_v}}{I}.
          \label{eq:bare-hull-w_d}
\end{align}
\end{subequations} 

\niceparagraph{Discussion.} We discuss the assumptions underlying the above model and its limitations. 
%The model is restricted to planar motion. %In particular, heave, roll, and pitch are neglected, as are other out-of-plane boom deformations. 
Hydrodynamic loads are represented via lumped, speed-dependent drag terms, while unmodeled environmental effects (e.g., wind, currents, and higher-order fluid-structure interaction) are neglected. The boom-vessel coupling enters the vessel EOM only through the tension vector $\mathbf f_l$, which we assume is measurable at the attachment point (e.g., via a load cell). 
(We emphasize that in the experimental results, we simulate the dynamics of the full model including all the individual links.)  %The boom itself is modeled as a planar chain of rigid links connected by pin joints with linear spring-damper interactions. % thus, detailed 3D geometry, buoyancy-induced catenary effects, and complex hydrodynamic coupling are not captured. 

%Despite those limitations, this model is sufficient to our purpose. For guidance and low-to-moderate-speed maneuvering of marine surface craft, the standard 3-DOF horizontal-plane model is widely used and captures the dominant dynamics whenever vertical motions are small compared to the planar motion \cite{Fossen2011}. Moreover, 
Prior boom-towing modeling and experiments adopt similar planar abstractions and demonstrate that they are adequate for closed-loop containment maneuvers \cite{pereda2011towards,giron2015preparing}, yet they do not describe the modeling of the forces between links. %\color{black}.\kiril{please clarify the difference, if any, between our work and those paper. The question is whether we present the model only for completeness (as it was described in previous work), in which case, we'll move it to the appendix, or bacause it provides additional value. }  %Since our focus is on multi-spill allocation, obstacle-aware planning, and trajectory tracking under measured boom forces, this planar model provides an appropriate fidelity-complexity trade-off for validating the proposed methods.

\section{Path tracking}\label{sec:PathTracking}
In this section, we introduce our paradigm for controlling individual ASV duos to execute the solution derived by the routing layer (Sec.~\ref{sec:routing}) to address Problem \ref{prob:control}.

%Next, we leverage the structure of our dynamical model (Eq.~\eqref{eq:disturbance-hull}) to design path tracking approaches.

\subsection{From Routing to Control}
\label{sec:SetpointsTracking}
Next, we leverage the structure of our dynamical model (Eq.~\eqref{eq:disturbance-hull}) to design feedback controllers.
Our objective is \emph{path following} of each vessel $i\in \{1,2\}$ in the horizontal plane, i.e., to drive $(x_i(t),y_i(t),\theta_i(t))$ to a path reference.  %$(x_i^\star(t),y_i^\star(t),\theta_i^\star(t))$.
The system is nonlinear and \emph{dynamically coupled} in two senses: (i) surge-sway-yaw DoFs are coupled through the rigid-body equations (Eq.~\eqref{eq:disturbance-hull}), and (ii) vessel and boom dynamics are coupled through the measured boom tension $\mathbf f_{l}$ applied at the tow point. The controller outputs are the thrust force $F_i$ and steering angle $\eta_i$. %, as they appear in Eq. \eqref{eq:disturbance-hull}.

To use such a controller for \emph{path following}, we convert the route obtained from the solution of the DMR problem into reference signals for the controlled DoFs (surge velocity $u$ and yaw/orientation $\theta$). Concretely, we discretize the path into a sequence of setpoints (Fig. ~\ref{fig:setpoints-trajectory}), producing piecewise-constant reference commands, where the two towing vessels are commanded to follow \textit{offset} paths whose separation is strictly smaller than the boom length.

When one vessel reaches its setpoint, its surge reference is set to zero so it holds position until the other vessel arrives, preventing the inter-vessel distance from exceeding the boom length. Once both vessels are aligned, their surge references are set to a constant cruising value.  This strategy motivates regulating surge velocity rather than position: setting  $u_{\mathrm{ref}}=0$ avoids aggressive orientation corrections that could violate the boom constraint. In contrast, direct position control may generate large transients and actuator saturation for large errors, whereas bounded velocity references limit this effect. 

%\label{sec:Controller-Definition}

\subsection{Feedback control}
\label{sec:Controller-Definition}
To enforce the above path-tracking rationale, we consider two approaches.

\niceparagraph{PID.} As a baseline, we implement a standard \pid controller for an individual boat with the following structure:
\begin{align}
F(s) &= \left(K_{p,u} + K_{i,u}\frac{1}{s} + \frac{K_{d,u}s}{\tau_{u}s + 1}\right) E_u(s), \\
\eta(s) &= \left(K_{p,\theta} + K_{i,\theta}\frac{1}{s} + \frac{K_{d,\theta}s}{\tau_{\theta}s + 1}\right) E_\theta(s),
\end{align}
Here, $F(s)$ and $\eta(s)$ are the Laplace transforms of the propulsion force and steering angle, and $E_u(s)$ and $E_\theta(s)$ are those of the surge velocity and orientation errors, respectively.
Each actuator is tuned by four parameters:
$K_p$, $K_i$, and $K_d$ set the proportional, integral, and derivative actions, and $\tau$ is for the derivative low-pass filter (for causality). 

\niceparagraph{Feedback linearization.} As a model-aware approach, we employ feedback linearization (\fbl)  with virtual control terms~\cite{Khalil2002NonlinearSystems} to decouple the nonlinear 3-DoF vessel dynamics into two manageable linear DoFs: surge velocity $u$ and yaw orientation $\theta$. 
Both of these DoFs are controlled via a lead controller. The third DoF, sway velocity $v$, is left uncontrolled. From a theoretical standpoint  this controller guarantees, under piecewise constant references,
$
u(t)-u_{\mathrm{ref}}\to 0, \theta(t)-\theta_{\mathrm{ref}}\to 0$, and  $v(t)\to 0$, as $t\to\infty$ (proof in Sec. ~\ref{sec:Guarantees}). Moreover, as we will see below, it requires only \emph{4} tuning parameters.

Starting from the surge $u$  and yaw $\omega$ equations of dynamics in Eq.~\eqref{eq:disturbance-hull}, the inputs $F(t)$ and $\eta(t)$, act through the body-frame components $F(t)\cos\eta(t)$ (surge) and $F(t)\sin\eta(t)$ (yaw). 

We split each actuation component into a virtual control,  $\alpha_u(t)$ and $\alpha_\omega(t)$, and a lumped term, $d_u(t)$ and $d_\omega(t)$, capturing known/estimated disturbances in  Eq.~\eqref{eq:disturbance-hull}:
\begin{subequations}
\label{eq:force_steering_deconstruct}
\begin{flalign}
F(t)\cos(\eta(t)) & = \alpha_u(t) + d_u(t),   \label{eq:F_cos}\\%[4pt]
F(t)\sin(\eta(t)) & = \alpha_\omega(t) + d_\omega(t).   \label{eq:F_sin}
\end{flalign}
\end{subequations}

To make those equations explicit, we define the lumped terms $d_u(t)$ and $d_\omega(t)$ by collecting all non-actuation contributions from the surge and yaw equations (Eq.  \eqref{eq:disturbance-hull}), i.e., the projections of the boom tension and the hydrodynamic/coriolis terms. This yields
\begin{subequations}
\label{eq:disturbance_virtual_components}
\begin{flalign}
d_u(t) & = -\,\mathbf f_l(t)\!\cdot\!\mathbf e_u(t)
          + \mu_l\!\big(u(t)\big)\,u(t) - m\,\omega(t)\,v(t),  \label{eq:u_dist} \\  
d_\omega(t) & = \mathbf f_l(t)\!\cdot\!\mathbf e_v(t)
          + \dfrac{\mu_\omega\!\big(\omega(t)\big)\,\omega(t)}{r}. \label{eq:omega_dist}
\end{flalign}
\end{subequations}
Substituting Eq. \eqref{eq:force_steering_deconstruct} and \eqref{eq:disturbance_virtual_components} into Eq. \eqref{eq:disturbance-hull}, we obtain two linear and one nonlinear EOMs, respectively:
\begin{subequations}
\label{eq:Linear_EOMs}
\begin{flalign}
\dot{u}(t) & = \frac{\alpha_u(t)}{m},   \quad
\dot{\omega}(t) = \ddot{\theta}(t)  = \frac{r}{I} \cdot \alpha_\omega (t), \label{eq:u_w_dot} \\
\dot{v}(t) & = \frac{ -\alpha_\omega(t) - \mu_{t}(v(t))v(t) - \frac{\mu_\omega(\omega (t))\omega (t)}{r}}{m} - \omega(t) u(t). \label{eq:v_dot}
\end{flalign}
\end{subequations}

Given the  parameters $ \alpha_u, \alpha_\omega, d_u$,  
and $ d_\omega $ we reconstruct the required propeller force and steering angle from Eq. ~\eqref{eq:F_cos}-\eqref{eq:F_sin},
\begin{align*}
    F &= \pm \sqrt{(\alpha_u + d_u) ^ 2 + (\alpha_\omega + d_\omega) ^ 2}, \\
   \eta &= \operatorname{atan2} \left(\frac{\alpha_\omega + d_\omega}{F}, \frac{\alpha_u + d_u}{F}\right),
\end{align*}
where the sign of $F$ is selected to be positive, so that the steering angle stays within $ -\frac{\pi}{2} \leq\eta \leq \frac{\pi}{2} $, ensuring the propeller continues to push from behind the hull.

% ------------------------------------------------------------

%\subsubsection{Controller and Reference Design}
% \label{sec:Controller-Design}

Now we can address the two linear systems (Eq. \eqref{eq:u_w_dot})  with the virtual control terms $\alpha_u$ and $\alpha_\omega$ as input, and we can design those terms to achieve a desired behavior. For the angular velocity $\dot{\theta}$, we choose to control the orientation, so the plant is of the form of a double integrator:
\[
    \ddot{\theta}(t) = \gamma_\omega \, \alpha_\omega(t), \quad \gamma_\omega \, = \,  \frac{r}{I} > 0.
\]
For the surge $u$, we choose to control the speed, so the plant is also of the form of a single integrator:
\[
    \dot{u}(t) = \gamma_u \, \alpha_u(t), \quad \gamma_u \, = \, \frac{1}{m} > 0.
\]

In terms of control in the Laplace domain, denote by $\Theta(s)$ and $U(s)$ the Laplace transforms of $\theta(t)$ and $u(t)$, respectively. In addition, $\alpha_\omega(s)$ and $\alpha_u (s)$ are the Laplace transforms of  $\alpha_\omega(t)$ and $\alpha_u(t)$, respectively. Hence, the above equations can be written as 

\begin{equation}
   \label{eq:yaw_plant}
   \frac{\Theta (s)}{\alpha_{\omega}(s)}=\frac{\gamma_{\omega}}{s^2}    
\end{equation}
\begin{equation}
    \label{eq:surge_plant}
    \frac{U(s)}{\alpha_u (s)} = \frac{\gamma_u}{s}
\end{equation}

Both plants in Equations ~\eqref{eq:yaw_plant}, \eqref{eq:surge_plant} can be controlled in a closed loop with a lead controller
\begin{equation}
\label{eq:LeadController}
    C(s) = K_p \cdot \frac{\sqrt{\beta} \cdot s + \Omega_c}{s + \sqrt{\beta} \cdot \Omega_c}
\end{equation}

Here, $\Omega_c$ denotes the desired crossover frequency, setting the transient speed. 
We choose $K_p = \frac{\Omega_c^2}{\gamma}$ to set the desired crossover frequency, and set
$\beta > 1$ to provide the required phase lead (phase margin). In addition, we implement this controller via a normalized feedback form, which preserves the stability and improves the transient response (See App.~\ref{app:NormalizedLeadController}).

\subsection{Theoretical guarantees of feedback-linearization controller}
\label{sec:Guarantees}
We prove that the tracking error of the \fbl controller tends to zero.  

\begin{theorem} [Asymptotic stability]
\label{thm:asymp-stability}
Consider the feedback-linearized closed-loop dynamics in Equations~\eqref{eq:Linear_EOMs},
% \begin{align}
% \dot u(t) &= \frac{1}{m}\,\alpha_u(t), \label{eq:lin_u}\\
% \dot\theta(t) &= \omega(t), \label{eq:kin_theta}\\
% \dot\omega(t) &= \frac{r}{I}\,\alpha_\omega(t), \label{eq:lin_omega}\\
% \dot v(t) &= \frac{-\alpha_\omega(t)-\mu_t(v(t))v(t)-\mu_\omega(\omega(t))\omega(t)/r}{m}
%             -\omega(t)u(t). \label{eq:lin_v}
% \end{align}
with the virtual inputs $\alpha_u,\alpha_\omega$ produced by the (normalized)
lead controllers
\begin{align}
C_u(s) &= K_{p_u}\,\frac{\sqrt{\beta_u}\,s+\Omega_{c_u}}{s+\sqrt{\beta_u}\,\Omega_{c_u}},
&K_{p_u}=\frac{\Omega_{c_u}}{\gamma_u},\quad \gamma_u=\frac{1}{m}, \label{eq:Cu}\\
C_\omega(s) &= K_{p_\omega}\,\frac{\sqrt{\beta_\omega}\,s+\Omega_{c_\omega}}{s+\sqrt{\beta_\omega}\,\Omega_{c_\omega}},
&K_{p_\omega}=\frac{\Omega_{c_\omega}^{2}}{\gamma_\omega},\quad \gamma_\omega=\frac{r}{I}. \label{eq:Cw}
\end{align}
% implemented in the normalized-feedback form of
% Section~\ref{sec:normalized-lead_controller} (which preserves the
% closed-loop characteristic polynomials).
Additionally,  assume the following:
\begin{enumerate}
\item[(A1)] the values $u_{\rm ref}$ and $\theta_{\rm ref}$ are constants;
\item[(A2)] the conditions $\Omega_{c_u}>0$, $\beta_u>0$, $\Omega_{c_\omega}>0$, and $\beta_\omega>1$ are met; 
%(so the surge and yaw characteristic polynomials proven in
%\ref{Characteristic_polynomial} are Hurwitz);
\item[(A3)] the sway-drag satisfies $\mu_t(v)\ge \underline\mu>0$ for all $v\neq 0$;
\item[(A4)] the boom-tension $\mathbf f_l$ and drag terms $\mu_l(u) \cdot u, \, \mu_t(v) \cdot v, \, \mu_\omega(\omega) \cdot \omega$ 
are bounded.
\end{enumerate}
Then, for any initial condition, the tracking errors satisfy
$
u(t)-u_{\mathrm{ref}}\to 0, \theta(t)-\theta_{\mathrm{ref}}\to 0$, and  $v(t)\to 0$, $
\qquad\text{as }t\to\infty$.
\end{theorem}

Assumption A2 is solely based on the formation of a lead controller, and Assumption A3 is based on the non-negativity of the drag function $\mu_t(v)$ as defined in Sec. ~\ref{sec:BoatDynamics}.
We obtain the proof through the following claims. 

%\begin{proof} We break the argument into three claims.

\begin{claim}[Surge and yaw tracking are exponentially stable]
\label{clm:uy_exp}
Under Assumptions (A1)-(A2), the surge loop $(u,\alpha_u)$ and the yaw loop $(\theta,\alpha_\omega)$
are exponentially stable about the equilibrium induced by the  constant references $u_{ref}$ and $\theta_{ref}$. 
% , i.e., for some constants $c_1,c_2>0$ and time constants $T_1,T_2>0$,  both the surge velocity error and the orientation error can be expressed as follows:
% \[|u(t)-u_{ref}| \leq c_1 \cdot e^{-t/T_1}, |\theta(t)-\theta_{ref}| \leq c_2 \cdot e^{-t/T_2}.\]
\end{claim}
\begin{proof}
\label{Surge-Yaw-Proof}
We derive the characteristic polynomials for the surge and yaw loops and prove that they are Hurwitz under (A1) and (A2). This implies that the corresponding linear closed-loop dynamics are exponentially stable and track constant references with zero steady-state error.

For the surge loop, with open-loop transfer function \(L_u(s)=C_u(s)P_u(s)\), the
closed-loop characteristic equation \(1+L_u(s)=0\) yields
\[
\chi_u(s)=s^{2}+2\sqrt{\beta_u}\,\Omega_{c_u}s+\Omega_{c_u}^{2}.
\]
For a quadratic \(p(s)=a_2 s^2+a_1 s+a_0\), the Routh-Hurwitz criterion \cite{Routh1877} reduces to
\(a_2>0,\;a_1>0,\;a_0>0\). Here \(a_2=1\), \(a_1=2\sqrt{\beta_u}\,\Omega_{c_u}\),
and \(a_0=\Omega_{c_u}^{2}\), which are strictly positive for \(\Omega_{c_u}>0\)
and \(\beta_u>0\); hence the surge closed-loop polynomial is Hurwitz
\cite{Hurwitz1895}.

For the yaw loop, with \(L_\omega(s)=C_\omega(s)P_\omega(s)\), the closed-loop
equation \(1+L_\omega(s)=0\) gives the characteristic polynomial
\[
\chi_\omega(s)=
s^{3}
+ \sqrt{\beta_\omega}\,\Omega_{c_\omega} s^{2}
+ \sqrt{\beta_\omega}\,\Omega_{c_\omega}^{2} s
+ \Omega_{c_\omega}^{3}.
\]
For a cubic \(s^{3}+a_{2}s^{2}+a_{1}s+a_{0}\), the Routh--Hurwitz conditions are
\(a_{2}>0,\;a_{1}>0,\;a_{0}>0,\;a_{2}a_{1}>a_{0}\). %\cite{Routh1877,Hurwitz1895}.
Here
\[
a_{2}=\sqrt{\beta_\omega}\,\Omega_{c_\omega},\quad
a_{1}=\sqrt{\beta_\omega}\,\Omega_{c_\omega}^{2},\quad
a_{0}=\Omega_{c_\omega}^{3},
\]
which are positive for \(\Omega_{c_\omega}>0\) and \(\beta_\omega>0\), and
\[
a_{2}a_{1}=\beta_\omega\,\Omega_{c_\omega}^{3}> \Omega_{c_\omega}^{3}=a_{0}
\quad\Longleftrightarrow\quad \beta_\omega>1.
\]
% Therefore, under \(\Omega_{c_u}>0,\beta_u>0\) and \(\Omega_{c_\omega}>0,\beta_\omega>1\),
% both the surge and yaw closed-loop characteristic polynomials are Hurwitz, and the
% corresponding tracking loops are asymptotically stable.
\end{proof}
% \emph{Justification.} The characteristic polynomials derived in
% \ref{Characteristic_polynomial} are Hurwitz under (A2); 

Given that the conditions in Claim \ref{clm:uy_exp} are satisfied, we can now address the third DOF of the sway velocity $v(t)$.

\begin{claim}[Sway stability under exponentially decaying disturbance]
\label{clm:sway_gas}
Consider the sway dynamics under feedback linearization:
\begin{equation}
\label{eq:sway_claim_dynamics}
\dot v
= -\frac{1}{m}\,\mu_t(v)\,v + \phi(t),
\end{equation}
where
\begin{equation}
\label{eq:sway_phi_def}
\phi(t)
\;=\;
\frac{-\alpha_\omega(t)\;-\;\mu_{\omega}(\omega(t))\,\omega(t)/r}{m}
\;-\;\omega(t)\,u(t).
\end{equation}
Assume:
\begin{enumerate}
\item[(S1)] There exists $\underline\mu>0$ such that $\mu_t(v)\ge \underline\mu$ for all $v\neq 0$;
\item[(S2)] The function $\phi(t)$ decays exponentially: there exist $\bar\phi>0$ and $\lambda>0$ such that
\begin{equation}
\label{eq:sway_phi_decay}
|\phi(t)| \le \bar\phi e^{-\lambda t},\qquad \forall\, t\ge 0.
\end{equation}
\end{enumerate}
Then $v(t)\to 0$ as $t\to\infty$. Moreover, $v(t)$ converges exponentially.
\end{claim}

\begin{proof}
Choose the quadratic Lyapunov function
\[
V(v)=\tfrac12\,m\,v^2,
\]
which is positive definite and radially unbounded. Differentiating along
Equation \eqref{eq:sway_claim_dynamics} yields
\[
\dot V
= m v \dot v
= -\mu_t(v)\,v^2 + m v \phi(t)
\le -\underline\mu v^2 + m|v|\cdot |\phi(t)|.
\]
Using $v^2=\tfrac{2}{m}V$ and $|v|=\sqrt{\tfrac{2}{m}}\sqrt{V}$ yields the ISS-type bound
\begin{equation}
\label{eq:sway_iss_ineq}
\dot V
\le
-\frac{2\underline\mu}{m}\,V
+\sqrt{2m}\,\sqrt{V}\,|\phi(t)|.
\end{equation}
By standard ISS comparison arguments (e.g., Proposition~4.19 in
Khalil~\cite{Khalil2002NonlinearSystems}), there exist class-$\mathcal{KL}$ and
class-$\mathcal{K}$ functions $\beta$ and $\gamma$ such that
\begin{equation}
\label{eq:sway_iss_bound}
|v(t)|
\le
\beta\!\bigl(|v(0)|,t\bigr)
+
\gamma\!\Bigl(\sup_{0\le \tau\le t}|\phi(\tau)|\Bigr),
\qquad \forall\, t\ge 0.
\end{equation}
Using Eq. \eqref{eq:sway_phi_decay}, we have
$\sup_{0\le \tau\le t}|\phi(\tau)| \le \bar\phi$ for all $t$, and
$\sup_{0\le \tau\le t}|\phi(\tau)| \to 0$ as $t\to\infty$.
Therefore, \eqref{eq:sway_iss_bound} implies $|v(t)|\to 0$. Furthermore, combining Equations \eqref{eq:sway_phi_decay} and \eqref{eq:sway_iss_bound},
yields an explicit exponential estimate of the form
\[
|v(t)|
\le
k_1 |v(0)|e^{-c t}
+
k_2 \bar\phi e^{-\lambda t},
\qquad k_1,k_2,c>0,
\]
so $v(t)$ converges exponentially to zero.
\end{proof}

Claims~\ref{clm:uy_exp} and  \ref{clm:sway_gas} imply
$u(t)\to u_{\rm ref}$, $\theta(t)\to\theta_{\rm ref}$, and $v(t)\to 0$, completing the proof of Theorem \ref{thm:asymp-stability}.

}{\section{Boom-Towing ASV Duo Control}
\label{sec:feedback}
In this section, we introduce our paradigm for controlling individual ASV duos to execute the solution derived by the routing layer (Sec.~\ref{sec:routing}) to address Problem \ref{prob:control}.

\subsection{Coupled Vessel–Boom Dynamics}
\label{sec:Dynamics}
We first describe the dynamics of an ASV duo towing a boom. 
Building on simplified single-ASV models ~\cite{Fossen2011}, we capture boom effects as an external load applied at each vessel’s tow point. 

\begin{wrapfigure}{R}{0.4\columnwidth}
   \centering
   \vspace{-10pt}
\includegraphics[width=0.4\columnwidth]
{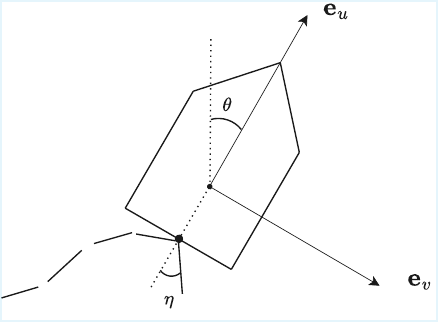}
\end{wrapfigure}
%On the right, the body-fixed 
The body-frame surge–sway axes are centered at the vessel CoM, with
surge and sway velocities $u$ and $v$. 
The inputs are the thrust $F$ and steering (rudder) angle $\eta$. The planar rigid-body EoM are:  
\begin{subequations}
\label{eq:disturbance-hull}
\begin{align}
\dot{u} &= \frac{F\cos\eta \;-\;\mu_{l}(u)\,u \; + \; \mathbf{f}_l \cdot \mathbf{e_u}}{m}
          \;+\; \omega\,v, \label{eq:bare-hull-u_d} \\[4pt]
\dot{v} &= \frac{-\,F\sin\eta \;-\;\mu_{t}(v)\,v \; + \; \mathbf{f}_l \cdot \mathbf{e_v}}{m}
          \;-\; \omega\,u, \label{eq:bare-hull-v_d} \\[4pt]
\dot{\omega} &= \frac{r\,F\sin\eta \;-\;\mu_{\omega}(\omega)\,\omega \; - \; r \cdot \mathbf{f}_l \cdot \mathbf{e_v}}{I}.
          \label{eq:bare-hull-w_d}
\end{align}
\end{subequations} 
Here, %\(\omega = \dot{\theta}\) 
$\omega$
is the yaw rate, $m$ is the mass, $I$ is the planar moment of inertia about the CoM, \(r\) is the propeller offset from the CoM, and \(\mu_{l},\mu_{t},\mu_{r}\) are the nonlinear drag coefficients that depend on the corresponding speed components, respectively: $
    \mu_* (f(t)) = \kappa_* \cdot |f(t)|$ for some $\kappa_* > 0$. %where $i \in \{l, t, r\}$.
% $ \mu_i (f(t)) = \kappa_i \cdot |f(t)|, \quad \kappa_i > 0, \quad i \in \{l, t, r\}$. 
The cross-coupling terms \(\omega v\) and \(-\omega u\) %in Eq.~\eqref{eq:bare-hull-v_d} and~\eqref{eq:bare-hull-u_d} 
represent Coriolis effects in the body frame.

The vector $\mathbf f_l$ denotes the measurable boom tension at the vessel's tow point, expressed in the body frame. 
In simulation, the boom is modeled as a planar articulated chain of identical rigid links connected by pin joints with a spring-damper model.

%Figure~\ref{fig:link-frame} describes the tangential-normal frame ($\mathbf{t,n}$) fixed to each link’s CoM.

% \niceparagraph{Discussion.} We discuss the assumptions underlying the above model and its limitations. 
% %The model is restricted to planar motion. %In particular, heave, roll, and pitch are neglected, as are other out-of-plane boom deformations. 
% Hydrodynamic loads are represented via lumped, speed-dependent drag terms, while unmodeled environmental effects (e.g., wind, currents, and higher-order fluid-structure interaction) are neglected. The boom-vessel coupling enters the vessel EoM only through the tension vector $\mathbf f_l$, which we assume is measurable at the attachment point (e.g., via a load cell). 
% (We emphasize that in the experimental results, we simulate the dynamics of the full model including all the individual links.) 

\subsection{From Routing to Control}
\label{sec:SetpointsTracking}
%In this section, we describe our approach for feedback control for a ASV duo. %We first discuss the dynamics of our system, then introduce our controller, and conclude this section with a theoretical evaluation. 

% \subsection{Control design}
% \label{sec: Feedback-Linearization-Control}

% Next, we leverage the structure of our dynamical model, specifically Equations~\eqref{eq:disturbance-hull}, to design a minimalist feedback controller. Our control challenge is to track a planned trajectory in a coupled \kiril{coupled in what sense?} nonlinear system. \kiril{Be more precise about the problem we consider here, e.g., point tracking and formally describe the controller input}
% We assume that we know each vessel's dynamical parameters ($m, I, r$ in \ref{eq:bare-hull}). We also assume to know the external forces and torques applied on each vessel, due to friction with the water ($\mu_i(f(t))$, also in \ref{eq:bare-hull}) and the tension force from the boom. Robust/adaptive control under model uncertainty is beyond the scope of this work . \kiril{discuss whether this is a reasonable assumption. Also, for clarity, state the variables for the forces, etc.}
% Our objective is \emph{pose tracking} of each vessel in the horizontal plane, i.e., to drive $(x_i(t),y_i(t),\psi_i(t))$ to a time-parameterized reference $(x_i^\star(t),y_i^\star(t),\psi_i^\star(t))$.

Next, we leverage the structure of our dynamical model (Eq.~\eqref{eq:disturbance-hull}) to design feedback controllers.
Our objective is \emph{path following} of each vessel $i\in \{1,2\}$ in the horizontal plane, i.e., to drive $(x_i(t),y_i(t),\theta_i(t))$ to a path reference.  %$(x_i^\star(t),y_i^\star(t),\theta_i^\star(t))$.
The system is nonlinear and \emph{dynamically coupled} in two senses: (i) surge-sway-yaw DoFs are coupled through the rigid-body equations (Eq.~\eqref{eq:disturbance-hull}), and (ii) vessel and boom dynamics are coupled through the measured boom tension $\mathbf f_{l}$ applied at the tow point. The controller outputs are the thrust force $F_i$ and steering angle $\eta_i$. %, as they appear in Eq. \eqref{eq:disturbance-hull}.

To use such a controller for \emph{path following}, we convert the route obtained from the solution of the DMR problem into reference signals for the controlled DoFs (surge velocity $u$ and yaw/orientation $\theta$). Concretely, we discretize the path into a sequence of setpoints (Fig. ~\ref{fig:setpoints-trajectory}), producing piecewise-constant reference commands, where the two towing vessels are commanded to follow \textit{offset} paths whose separation is strictly smaller than the boom length.

When one vessel reaches its setpoint, its surge reference is set to zero so it holds position until the other vessel arrives, preventing the inter-vessel distance from exceeding the boom length. Once both vessels are aligned, their surge references are set to a constant cruising value.  This strategy motivates regulating surge velocity rather than position: setting  $u_{\mathrm{ref}}=0$ avoids aggressive orientation corrections that could violate the boom constraint. In contrast, direct position control may generate large transients and actuator saturation for large errors, whereas bounded velocity references limit this effect.

\subsection{Controller design}
\label{sec:Controller-Definition}
To enforce the above path-tracking rationale, we consider two approaches.

\niceparagraph{PID.} As a baseline, we implement a standard \pid controller for an individual boat with the following structure:
\begin{align}
F(s) &= \left(K_{p,u} + K_{i,u}\frac{1}{s} + \frac{K_{d,u}s}{\tau_{u}s + 1}\right) E_u(s), \\
\eta(s) &= \left(K_{p,\theta} + K_{i,\theta}\frac{1}{s} + \frac{K_{d,\theta}s}{\tau_{\theta}s + 1}\right) E_\theta(s),
\end{align}
Here, $F(s)$ and $\eta(s)$ are the Laplace transforms of the propulsion force and steering angle, and $E_u(s)$ and $E_\theta(s)$ are those of the surge velocity and orientation errors, respectively.
Each actuator is tuned by four parameters:
$K_p$, $K_i$, and $K_d$ set the proportional, integral, and derivative actions, and $\tau$ is for the derivative low-pass filter (for causality). 

\niceparagraph{Feedback linearization.} As a model-aware approach, we employ feedback linearization (\fbl)  with virtual control terms~\cite{Khalil2002NonlinearSystems} to decouple the nonlinear 3-DoF vessel dynamics into two manageable linear DoFs: surge velocity $u$ and yaw orientation $\theta$. 
Both of these DoFs are controlled via a lead controller. The third DoF, sway velocity $v$, is left uncontrolled. From a theoretical standpoint  this controller guarantees, under piecewise constant references,
$
u(t)-u_{\mathrm{ref}}\to 0, \theta(t)-\theta_{\mathrm{ref}}\to 0$, and  $v(t)\to 0$, as $t\to\infty$ (proof omitted). Moreover, as we will see below, it requires only \emph{two} tuning parameters per vessel.

%and is proven to be asymptotically stable. %in \ref{sec:StabilityProof-Sway}.
%This approach provides theoretical stability guarantees while enabling independent control of each vessel, provided that the boom geometric constraints are satisfied through trajectory planning.

%(also known as state-feedback linearization): 
%In this approach, a nonlinear plant is algebraically transformed so that selected closed-loop channels obey linear error dynamics driven by auxiliary inputs (the virtual controls)~\cite{Khalil2002NonlinearSystems}. 

Starting from the surge $u$  and yaw $\omega$ equations of dynamics in Eq.~\eqref{eq:disturbance-hull}, the inputs $F(t)$ and $\eta(t)$, act through the body-frame components $F(t)\cos\eta(t)$ (surge) and $F(t)\sin\eta(t)$ (yaw). 
%
%\footnote{If the yaw equation uses a moment input, the corresponding term is $rF(t)\sin\eta(t)$; the decomposition below applies identically after including the lever arm $r$.}
We split each actuation component into a virtual control,  $\alpha_u(t)$ and $\alpha_\omega(t)$, and a lumped term, $d_u(t)$ and $d_\omega(t)$, capturing known/estimated disturbances in  Eq.~\eqref{eq:disturbance-hull}:
\begin{subequations}
\label{eq:force_steering_deconstruct}
\begin{flalign}
F(t)\cos(\eta(t)) & = \alpha_u(t) + d_u(t),   \label{eq:F_cos}\\%[4pt]
F(t)\sin(\eta(t)) & = \alpha_\omega(t) + d_\omega(t).   \label{eq:F_sin}
\end{flalign}
\end{subequations}

To make those equations explicit, we define the lumped terms $d_u(t)$ and $d_\omega(t)$ by collecting all non-actuation contributions from the surge and yaw equations (Eq.  \eqref{eq:disturbance-hull}), i.e., the projections of the boom tension and the hydrodynamic/coriolis terms. This yields
\begin{subequations}
\label{eq:disturbance_virtual_components}
\begin{flalign}
d_u(t) & = -\,\mathbf f_l(t)\!\cdot\!\mathbf e_u(t)
          + \mu_l\!\big(u(t)\big)\,u(t) - m\,\omega(t)\,v(t),  \label{eq:u_dist} \\  
d_\omega(t) & = \mathbf f_l(t)\!\cdot\!\mathbf e_v(t)
          + \dfrac{\mu_\omega\!\big(\omega(t)\big)\,\omega(t)}{r}. \label{eq:omega_dist}
\end{flalign}
\end{subequations}
Substituting Eq. \eqref{eq:force_steering_deconstruct} and \eqref{eq:disturbance_virtual_components} into Eq. \eqref{eq:disturbance-hull}, we obtain two linear and one nonlinear EoMs, respectively:
\begin{subequations}
\label{eq:Linear_EoMs}
\begin{flalign}
\dot{u}(t) & = \frac{\alpha_u(t)}{m},   \quad
\dot{\omega}(t) = \ddot{\theta}(t)  = \frac{r}{I} \cdot \alpha_\omega (t), \label{eq:u_w_dot} \\
\dot{v}(t) & = \frac{ -\alpha_\omega(t) - \mu_{t}(v(t))v(t) - \frac{\mu_\omega(\omega (t))\omega (t)}{r}}{m} - \omega(t) u(t). \label{eq:v_dot}
\end{flalign}
\end{subequations}

Given the  parameters $ \alpha_u, \alpha_\omega, d_u$,  %(see details on their computation in Sec.~\ref{sec:feedback:implementation}) 
and $ d_\omega $ %(computation shown in Eq.~\eqref{eq:disturbance_virtual_components}), 
we reconstruct the required propeller force and steering angle from Eq. ~\eqref{eq:F_cos}-\eqref{eq:F_sin},
\begin{align*}
    F &= \pm \sqrt{(\alpha_u + d_u) ^ 2 + (\alpha_\omega + d_\omega) ^ 2}, \\
   \eta &= \operatorname{atan2} \left(\frac{\alpha_\omega + d_\omega}{F}, \frac{\alpha_u + d_u}{F}\right),
\end{align*}
where the sign of $F$ is selected to be positive, so that the steering angle stays within $ -\frac{\pi}{2} \leq\eta \leq \frac{\pi}{2} $, ensuring the propeller continues to push from behind the hull.

% ------------------------------------------------------------

%\subsubsection{Controller and Reference Design}
% \label{sec:Controller-Design}

We now treat Eq.~\eqref{eq:u_w_dot} as two linear systems with inputs $\alpha_u$ and $\alpha_\omega$, and choose these virtual controls to achieve the desired behavior.
\iftoggle{arxiv}{
For the angular velocity $\dot{\theta}$, we choose to control the orientation, so the plant is of the form of a double integrator:
\[
    \ddot{\theta}(t) = \gamma_\omega \, \alpha_\omega(t), \quad \gamma_\omega \, = \,  \frac{r}{I} > 0.
\]
For the surge $u$, we choose to control the speed, so the plant is 
of the form of a single integrator:
\[
    \dot{u}(t) = \gamma_u \, \alpha_u(t), \quad \gamma_u \, = \, \frac{1}{m} > 0.
\]
In terms of control in the Laplace domain, denote by $\Theta(s)$ and $U(s)$ the Laplace transforms of $\theta(t)$ and $u(t)$, respectively. In addition, $\alpha_\omega(s)$ and $\alpha_u (s)$ are the Laplace transforms of  $\alpha_\omega(t)$ and $\alpha_u(t)$, respectively. Hence, the above equations can be written as 
\begin{equation}
   \frac{\Theta (s)}{\alpha_{\omega}(s)}=\frac{\gamma_{\omega}}{s^2}, \quad
    \frac{U(s)}{\alpha_u (s)} = \frac{\gamma_u}{s}.\label{eq:yaw_surge_plant}
\end{equation}
}
Those systems can be controlled in a closed loop with a lead controller for $u$ of the form
\begin{equation}
     C_u(s) = K_u \cdot \frac{\sqrt{\beta_u} \cdot s + \Omega_{c_u}}{s + \sqrt{\beta_u} \cdot \Omega_{c_u}}.\label{eq:lead_control}
\end{equation}
Here, $\Omega_{c_u}$ denotes the desired crossover frequency, setting the transient speed. 
We choose $K_u = \frac{\Omega_{c_u}^2}{\gamma}$ to set the desired crossover frequency, and set
$\beta_u > 1$ to provide the required phase lead (phase margin). A lead controller for $\omega$ is obtained in a similar manner.
}

\section{Experimental Evaluation}
In this section, we evaluate our routing (Sec.~\ref{sec:routing}) and path-tracking (Sec.~\ref{\iftoggle{arxiv}{sec:Controller-Definition}{sec:feedback}}) approaches. 

\label{sec:exper-eval}
\subsection{Damage Minimizing Routing Solution Evaluation}

%\niceparagraph{Experimental Setup.} 
To assess the routing framework, we generate synthetic multi-spill scenarios over a bounded planar workspace $W$. Static polygonal obstacles are randomly generated to impose navigation constraints. The depot location is fixed, and spill locations are sampled uniformly at random within the free space of $W$. Each spill is assigned a risk weight $R_i$ drawn independently from a predefined range to represent heterogeneous environmental severity.
Transit distances $d(S_i,S_j)$ are computed using occupancy-grid discretization of $W$ followed by obstacle-avoiding $A^*$ search between spill representatives (Sec.~\ref{Motion_Based_GR}), yielding a DMR instance.

We evaluate two spill instances, with 25, 50, and 100 spills, and vary the number of ASV duos from one to ten. These scales reflect and exceed realistic large-scale spill fragmentation observed in major incidents~\cite{Sun2016OilSlickMorphology}. For each configuration, experiments are run with a 300-second time limit on a laptop equipped with an Intel Core Ultra~9 185H CPU and 64\,GB RAM. As a MILP solver, we rely on Gurobi \cite{gurobi}, where configuration parameters are left at default settings.

\begin{figure}
    \centering
    \vspace{5pt}
    
    % First plot: 25 Spills
    \includegraphics[width=0.95\linewidth]{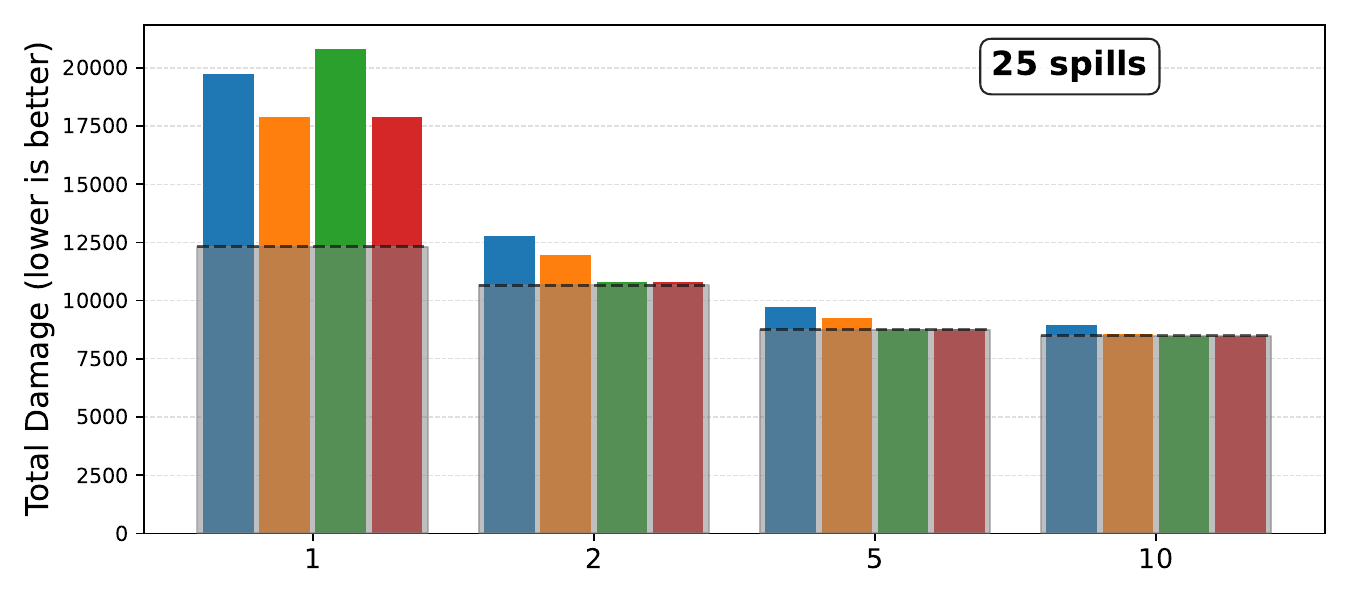}
    % \vspace{0.2cm} % Adjust vertical spacing between figures
    
    % Second plot: 50 Spills
    \includegraphics[width=0.95\linewidth]{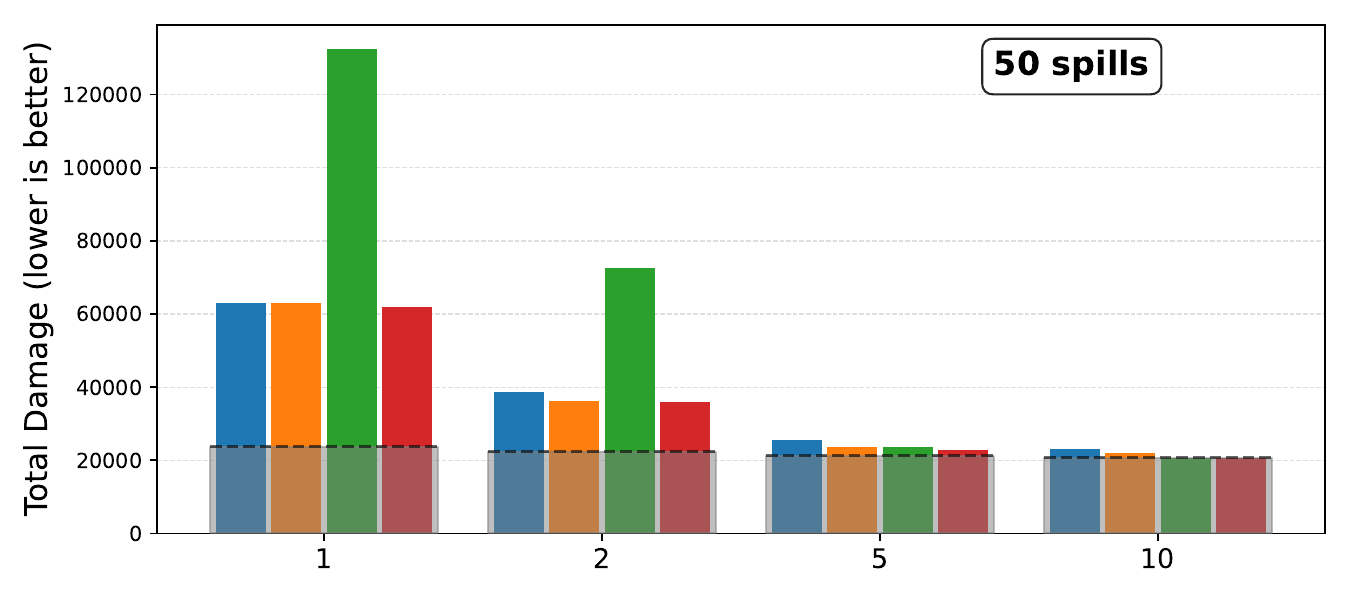}
    % \vspace{0.1cm}

    \includegraphics[width=0.95\linewidth]{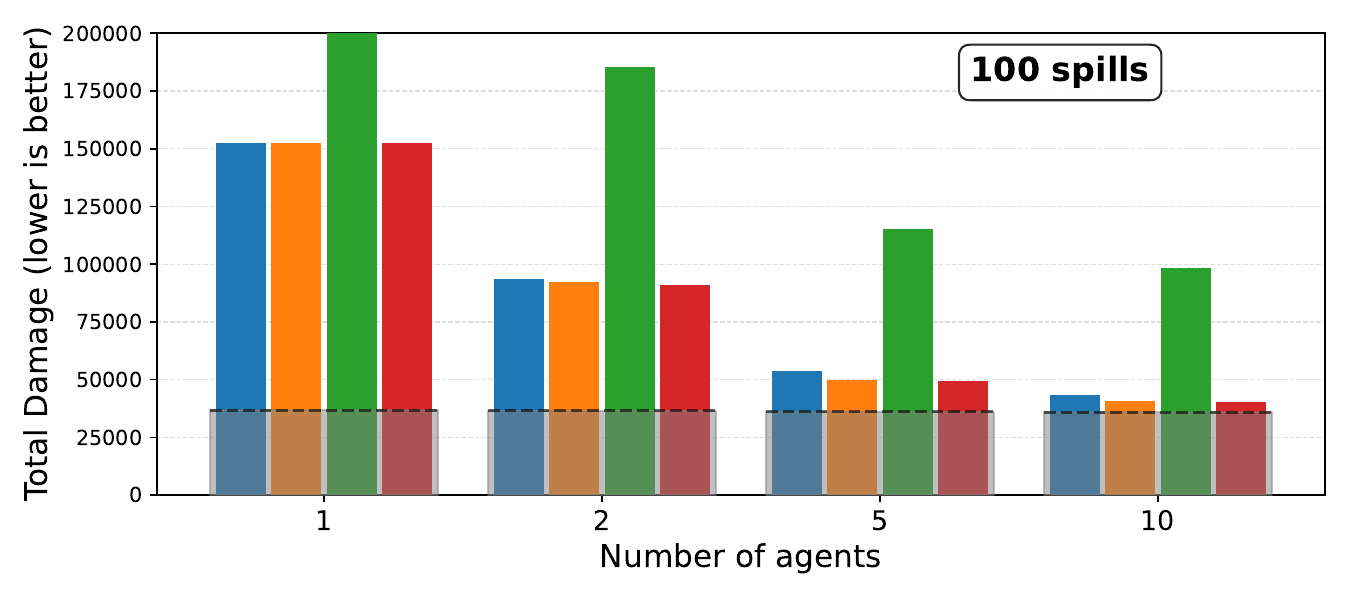}
    % \vspace{0.1cm}
    
    % Third plot: Legend
    \includegraphics[width=0.95\linewidth]{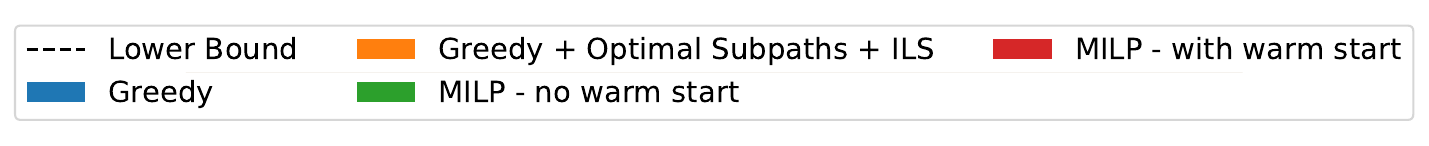}
    
    \caption{
    Final objective values for 25, 50, and 100 spills instances across varying numbers of agents. 
    Shaded gray regions indicate lower bounds. % obtained from the MILP branch-and-bound (BnB) solver within the 300-second time limit. 
    }
    \label{fig:route-eval}
\end{figure}

\niceparagraph{Solver Comparison.} Results for a representative set of scenarios are given in Fig.~\ref{fig:route-eval}. We evaluate four solution approaches: (i) \textit{Greedy} corresponds to our assignment-only heuristic (H1) without solving a MILP; (ii) \textit{Greedy + DP + ILS} denotes the full heuristic (Sec.~\ref{sec:method-heuristic}), again without invoking the MILP solver afterwards; 
(iii) \emph{MILP} denotes a pure MILP solution without using any warm-start heuristic; 
(iv) finally, \textit{MILP + warm-start} initializes the MILP solver with the full heuristic solution (H1-3). Lower bounds correspond to the best dual bounds obtained by the MILP (iv) branch-and-bound solver within the 300-second time limit.

For moderate agent counts (5 and 10 agents), the \textit{MILP + warm-start} solver reaches optimal solutions that match the computed lower bounds, or near optimality for 100 spills. Interestingly, the combinatorial difficulty decreases as the number of agents increases. %When the number of agents approaches the number of spills, routing becomes nearly trivial, as each agent services few spills. 
Conversely, for small agent counts, each agent must service more spills, increasing both assignment complexity and per-agent routing complexity.

For 25-spill instances, (near-)optimal solutions are obtained for two agents. For the 25-spill single-agent case, the DP solver (H2) recovers the optimal solution directly within the allotted time, while the MILP lower bound did not fully converge.
For more challenging configurations (e.g., 50 or 100 spills with 1–2 agents), the benefit of the warm-start heuristic becomes pronounced.  Initializing MILP with the DP-augmented heuristic significantly improves solution quality within the time limit compared to MILP alone, with accumulated damage of routing without warm start exceeding twice the damage with warm-start in the worst case. Moreover, the standalone heuristic consistently produces high-quality solutions, often within a minute of computation. 

Our greedy priority queue heuristic (H1),  consistently provides high-quality solutions in very short runtimes. The complete heuristic, applying single-agent DP and ILS refinement, can improve solution quality by up to $15\%$ with negligible additional runtime. The MILP solver, guided by this warm start, %usually provides only mild improvements over the heuristic solution within time limits, 
yet guarantees a small optimality gap in limited time, and in some cases improves the standalone heuristic solution by over $10\%$.
Overall, this demonstrates that we can effectively deal with realistic problem instances.

% ------------ 4. Simulations & Results --------------------------------
% \section{Simulations and Results}
% \input{Simulation_Results/Allocation_Results}

\subsection{Controller Evaluation}
\label{sec:control-results}
We evaluate the \pid and \fbl controllers (Sec. ~\ref{\iftoggle{arxiv}{sec:Controller-Definition}{sec:feedback}}) for path tracking under varying conditions. Across the tested scenarios, both controllers produce stable tracking along the planned path---indicating that the \fbl design, while requiring only half of the tuning parameters, captures the benefits of \pid in this setup. That said, the \pid consistently had an edge: it typically showed slightly lower steady-state errors. \fbl matches \pid’s behavior well but pays a modest performance penalty.

\begin{figure}[t]
    \centering
    \vspace{5pt}
    \includegraphics[width=0.45\textwidth]{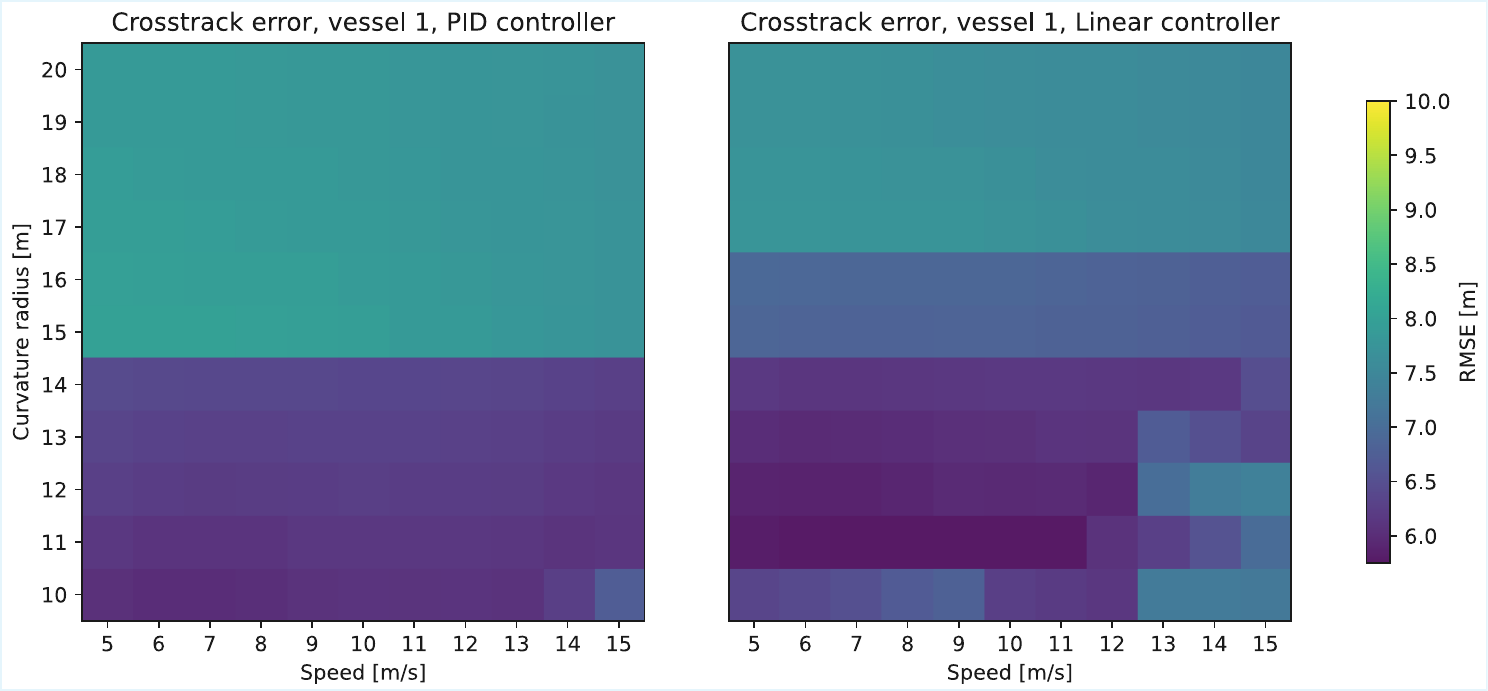}
    \includegraphics[width=0.45\textwidth]{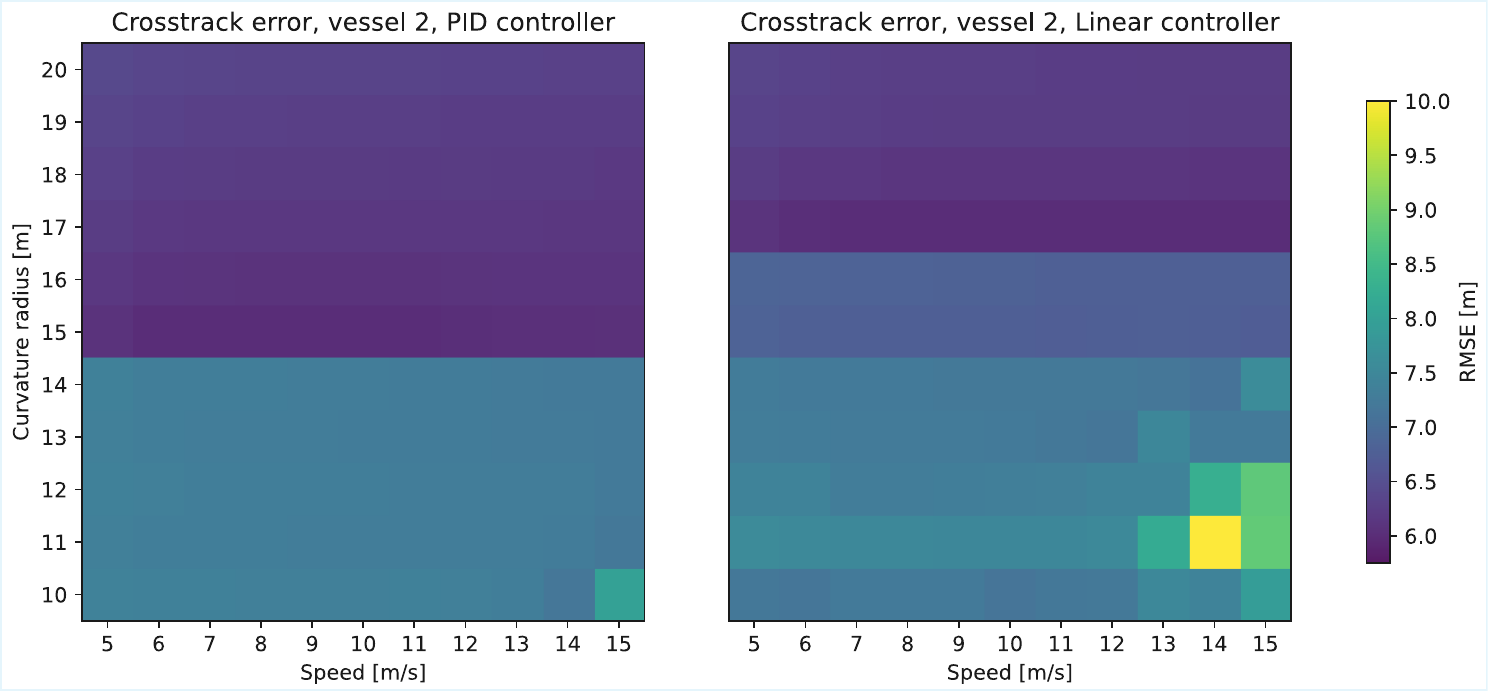}

    \includegraphics[width=0.45\textwidth]{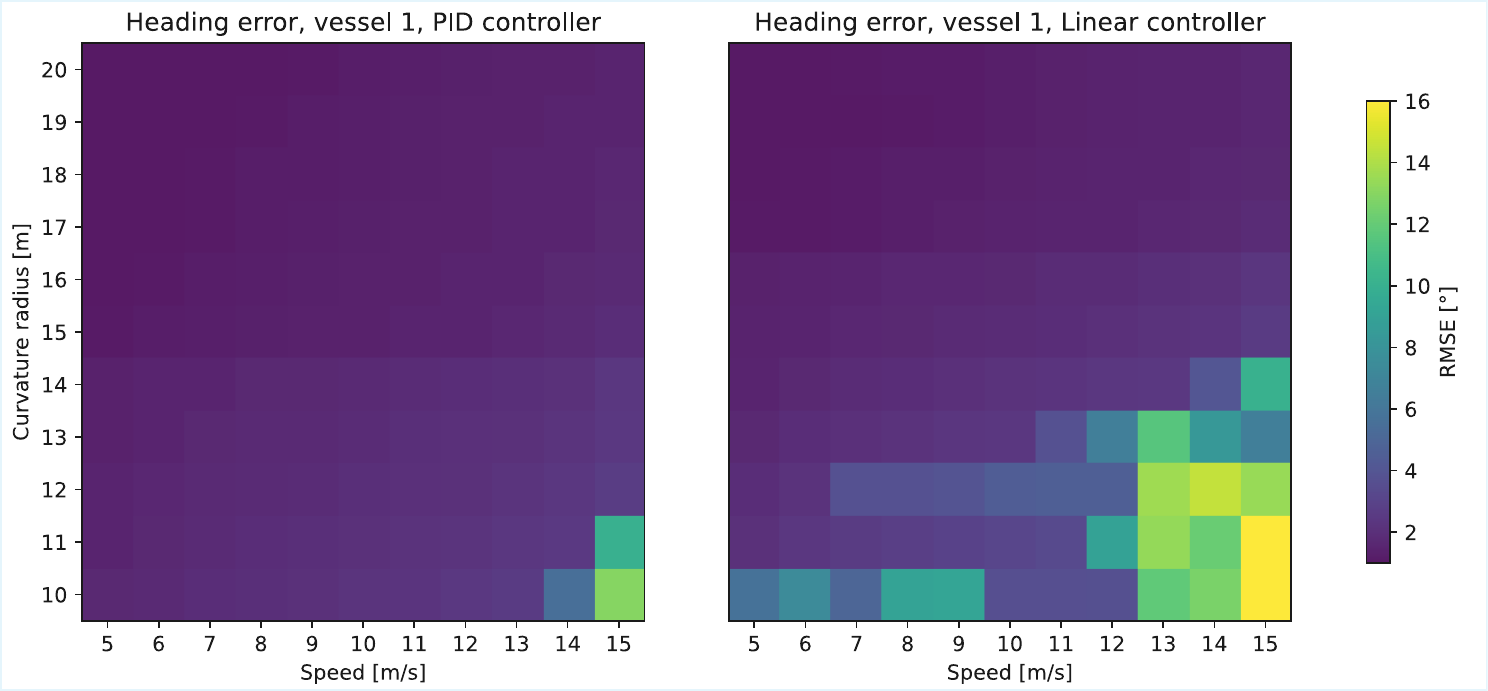}
    \includegraphics[width=0.45\textwidth]{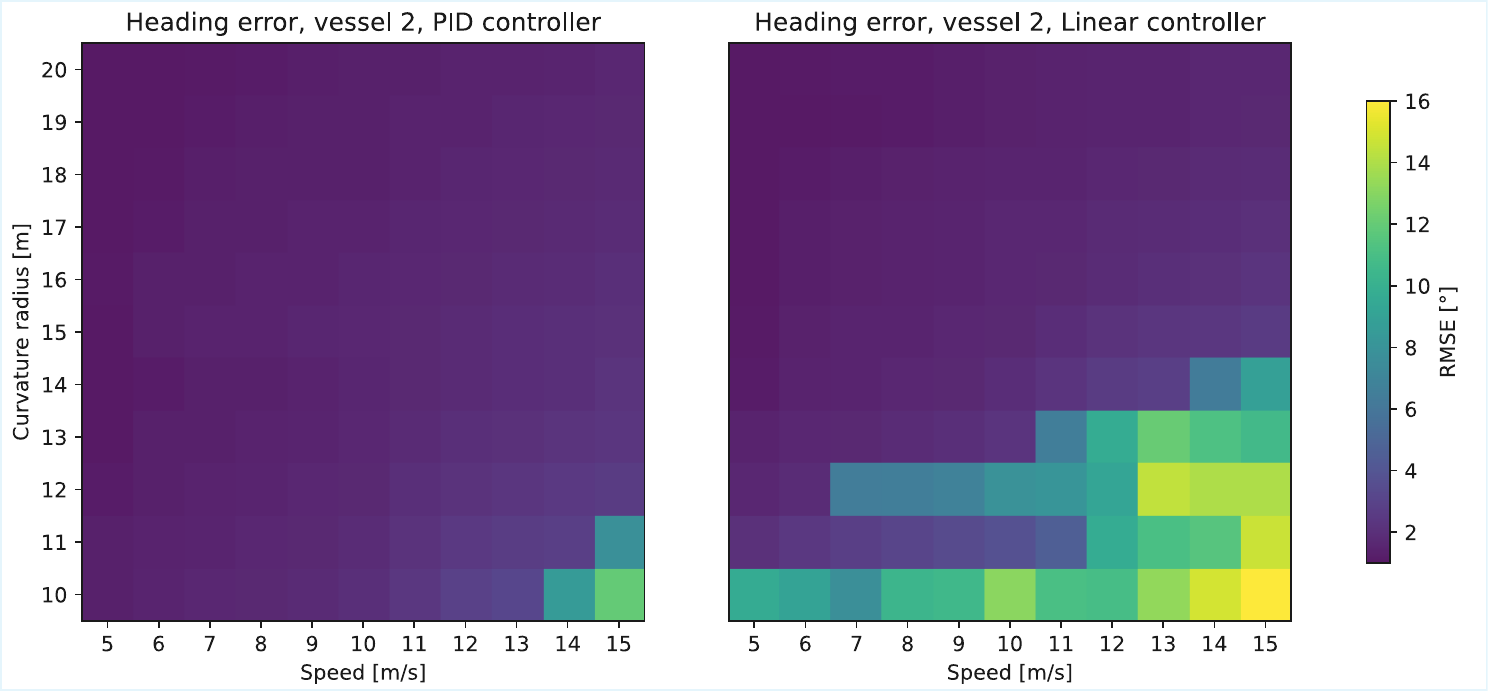}
    \caption{RMSE maps for cross-track (top) and heading errors (bottom) as functions of curvature radius and velocity for both vessels using feedback linearization control.}
    \label{fig:control-rmse-results}
\end{figure}

\niceparagraph{Controller setup.}
We tuned the \pid gains using an iterative trial-and-error procedure. Starting from a baseline (low $K_p$ with $K_i=K_d=0$), we increased $K_p$ until the response became sufficiently fast without sustaining
oscillations. Next, we introduced $K_i$ to eliminate steady-state error, and finally added a small $K_d$
(with derivative filtering for causality) to reduce overshoot. We stopped tuning once further adjustments
produced only marginal improvements\iftoggle{arxiv}{ in rise/settling time, overshoot, and steady-state error across the tested scenarios, while maintaining stable and consistent behavior.

}{. }
For the \fbl controller, the parameters $\beta_u$ and $\beta_{\omega}$ are selected to achieve a desired phase margin (through an explicit relation~\cite{ModernControl}) to improve stability. The crossover frequencies $\Omega_{c_u}$ and $\Omega_{c_{\omega}}$ are chosen (through trial and error) to decrease the settling time of the transient response. \iftoggle{arxiv}{ Specifically, the parameter $\beta_\omega$ is chosen to  satisfy a phase margin of $\phi_{lead_{\omega}}=60 ^\circ$. In a stable plant of a double integrator (Eq.~\eqref{eq:yaw_plant}), according to the Nyquist stability criterion ~\cite{ModernControl}, this is accomplished by the relation
\begin{equation}\label{eq:beta_phase}
\beta_\omega = \frac{1+\sin (\phi_{lead_{\omega}})}{1-\sin(\phi_{lead_{\omega}})}.
\end{equation}
For a plant of a single integrator (Eq.~\eqref{eq:surge_plant}), according to the Nyquist stability criterion, we already have a phase margin of $90^\circ$, thus we choose a phase margin of $\phi_{lead_{u}}=20^\circ$,  to slightly reduce overshoot. The value $\beta_u$ is derived similarly to Eq.~\eqref{eq:beta_phase}.}

\niceparagraph{Simulation model and parameters.}
An evaluation was conducted in simulation using the coupled vessel-boom model in Eq.~\eqref{eq:disturbance-hull}, where the boom enters as a measurable tow-point load $\textbf{f}_l$. The boom is simulated as an articulated chain of $n=40$ rigid links, with total length $L=40[m]$. The vessel parameters were selected to approximate a small ASV and were assumed to be identical for both vessels. \iftoggle{arxiv}{ 

For ease of implementation, the boom is initialized as a straight segment of length $L$, between the two boats. To allow the boom to actually enclose and retain oil, the two towing vessels are commanded to follow \textit{offset} paths whose separation is strictly smaller than $L$ (see Fig. ~\ref{fig:setpoints-trajectory}). If the inter-vessel distance were kept equal to the boom length, the boom would remain perpendicular to the direction of motion, effectively reducing to a taut straight line with negligible lateral containment; moreover, the resulting drag loads would be excessive and could practically damage the actuators.
Specifically, the mass and yaw inertia were set to $m = 600\,[kg]$ and $I = 500\,[kg\,m^2]$, 
the propulsor offset (lever arm) to $r = 2\, [m]$, and the drag coefficients to
$\kappa_\ell = 100 \,[\frac{N \cdot s^2}{m}]$, $\kappa_t = 10{,}000  [\frac{N \cdot s^2}{m}]$, and $\kappa_r = 1{,}000 \,[\frac{Nm\cdot s^2}{rad}]$.  %We have 40 links in each boom, of length $1 \, [m]$, each.

We assume calm-water conditions to isolate the tracking performance attributable to the controllers themselves.  The model neglects heave/roll/pitch, added-mass and wave-radiation effects, and spatially varying currents ~\cite{Fossen2011}; Thus, the reported results should be interpreted as a baseline comparison under a benign sea state, while incorporating wind/current/wave disturbances is left for future work.
}
{Full details on vessel parameters will be given in our open-source repository (upon acceptance). }

\niceparagraph{Test scenario.}
Each scenario consists of tracking a smooth reference path for an ASV duo.
Both vessels start at the same initial pose and orientation, separated by a fixed distance. %consistent with the boom length (i.e., initial separation $L$), 
They are commanded to follow the same path with synchronized setpoints (see Sec.~\ref{sec:SetpointsTracking}). The generated reference paths are Dubins paths~\cite{Dubin} between $(x_0,y_0,\theta_0)=(0,0,0)$ and $(x_f,y_f,\theta_f)=(100,65,\pi)$, as in Fig. ~\ref{fig:setpoints-trajectory}. We also sweep the curvature radius
$\rho \in [10,20] [m]$ and reference speed $v_{\mathrm{ref}} \in [5,15] [\frac{m}{s}]$ to span from mild to aggressive maneuvers.
The chosen initial and end points produce a trajectory long enough to exhibit both transient and steady-state behavior while remaining representative of local maneuvers used during the approach and encirclement of a spill.
For the results, vessel 1 is the left vessel and vessel 2 is the right vessel.

%\kiril{add a small visualization of one such example}

% Dubins path planning with varying maximum curvature constraints, while the commanded velocities were systematically varied across different simulation runs.

% \subsubsection{Simulation Setup}
% The simulation study examines the controller's performance across a parameter space defined by:
% \begin{itemize}
%     \item \textbf{Curvature radius}: $\rho \in [10, 20]$ meters, controlling the sharpness of turns in the Dubins path
%     \item \textbf{Reference velocity}: $v_{ref} \in [5, 15]$ m/s, representing different operational speeds
%     \item \textbf{Performance metrics}: 
%     \begin{itemize}
%         \item Cross-track error ($e_{cross}$): perpendicular distance from the vessel's actual position to the reference path.
%         \item Heading error ($e_{\theta}$): angular deviation between actual and reference vessel orientation.
%     \end{itemize}
% \end{itemize}

\niceparagraph{Results.}
For each ($\rho, \,v_{ref})$ pair,   Fig.~\ref{fig:control-rmse-results} reports the trajectory RMSE as our tracking metric. 

For cross track error (vessel~1), both controllers degrade with sharper turns (smaller $\rho$) and higher speeds.
Performance is mainly curvature-driven.
\iftoggle{arxiv}{For \fbl, this is apparent when the velocities are $12-15$ m/s and the curvature radius is $10-13$ m, and for \pid with velocities $13-14$ m/s and curvature $10$ m.}{}
Apart from the simulation where the turns are the sharpest and the velocities are high, \pid outperforms \fbl in curvature radii of up to $14$ m\iftoggle{arxiv}{ (error of up to $\approx 6.5$m vs an error of up to $\approx 7$m)}{}.
In the range of $15-16$m, \fbl outperforms \pid\iftoggle{arxiv}{  (error of $\approx 7 m$ vs $\approx 7.8 m$)}{}.
From curvature radii of $17$m, both controllers perform similarly\iftoggle{arxiv}{ with an error of $\approx 7.75 m$}{}. Similar trends are observed for vessel 2, although with higher error values due to a longer reference path.

For heading error, both controllers worsen at high speed and tight curvature. For \fbl, the "rough dynamics" region is substantially larger for both vessels\iftoggle{arxiv}{ ($\approx 20$ experiments)}{}, with high errors of up to $\approx 15 ^\circ$, while for \pid, this region is much smaller\iftoggle{arxiv}{ (3 experiments)}{} with high errors of up to $\approx 12 ^\circ$.
When the dynamics are less demanding, \fbl achieves heading errors of $\approx 5 ^\circ$ for both vessels, while \pid achieves heading errors of $\approx 3 ^\circ$ for both vessels.

% Overall, both controllers exhibit increased cross-track and heading errors with sharper turns and higher velocities, indicating the challenges of maintaining accurate path tracking under demanding path dynamics.
% Nevertheless, these errors are relatively negligible relative to the spatial scale of oil slicks and containment operations. % no significant downside given the  distances at sea and oil slick .

Overall, both controllers incur larger cross-track and heading errors at higher speeds and tighter turns, reflecting the increased difficulty of tracking aggressive trajectories. These errors remain small relative to the spatial scale of typical oil slicks and containment operations.

% From a theoretical perspective, \fbl ensures zero steady-state error, while \pid controller lacks such guarantees. In practice, both controllers completed their simulations, indicating they handled disturbances adequately. Overall, despite the theoretical benefits of \fbl, \pid achieved better results across all simulations. One clear practical disadvantage of the \pid approach is the large number of parameters (eight in total per vessel), which makes tuning particularly challenging. 
% In contrast, \fbl requires tuning only four parameters, which can be systematically achieved using well-established linear-control techniques\iftoggle{arxiv}{(lead controllers and frequency-response based tuning \cite{ModernControl})}{}. 

Theoretically, \fbl guarantees zero steady-state error, unlike \pid. In simulations, both tolerated different dynamics and completed all runs, but \pid consistently achieved slightly better tracking. Practically, \pid is harder to tune (eight gains per vessel), whereas \fbl requires only four parameters. % that can be tuned systematically via standard frequency-response methods\textbackslash{}iftoggle\{arxiv\}{ \textbackslash{}cite\{ModernControl\}}\{}. 

\section{Conclusion and Future Work}
In this work, we initiated the study of the problem of multiple oil-spill cleanup using boom-towing ASV duos. Our contribution lies in a careful modeling of the problem, and developing effective solution approaches for the corresponding routing and path tracking problems. 

In the future, we plan to validate the framework on real oil-spill data sets and evaluate our controllers under varying sea conditions (e.g., winds and currents), which we have neglected so far. From an algorithmic perspective, one issue that our current modeling of the routing problem overlooks is that once a spill is cleaned, other ASVs can pass through it, rather than around it, as we currently do. This might require a significant overhaul of our MILP approach, which we hope to explore in the future.   

\niceparagraph{Acknowledgments.} 
The AI system ChatGPT was used for light editing and grammar enhancement, as well as a preliminary literature review.

% ======================================================================
%  References
% ======================================================================
\bibliographystyle{IEEEtran}
\bibliography{references}

\iftoggle{arxiv} {
    \appendix
% \section{Normalized feedback implementation of a lead controller}
% \label{app:NormalizedLeadController}

\begin{figure*}[t]
  \centering
  \includegraphics[width=0.7\textwidth]{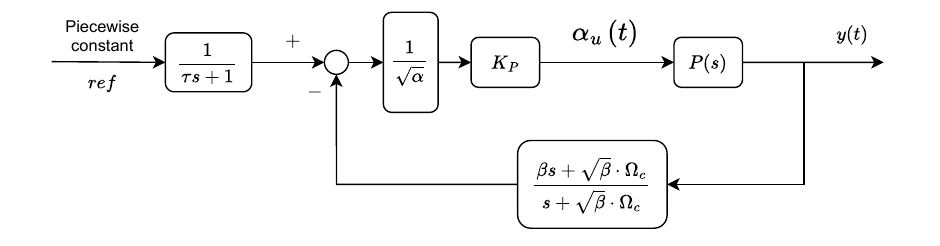}
  \caption{Lead controller in the normalized variant: the controller lies in the feedback while its static gain sits in the main branch along with the gain $K_P$. The reference signal is a piecewise constant signal, filtered with a low-pass filter to avoid discontinuous jumps in the control signal.}
  \label{fig:Lead-Controller}
\end{figure*}

\subsection{Normalized feedback implementation}
\label{app:NormalizedLeadController} 
We describe a normalized-feedback variant of \fbl.  The normalized lead controller from Eq.~\eqref{eq:LeadController}, wherein the static gain is equal to one, is placed in the feedback to reduce the overshoot of the system, while its static gain remains in the main branch. This is depicted in Fig. \ref{fig:Lead-Controller} and detailed below. This modification diminishes the effect of the lead-controller's zero in the complementary sensitivity function. 
Furthermore, we apply a first-order low-pass filter to the reference signal to avoid discontinuous jumps in the control signals applied to the vessel. 
This is especially important when switching the reference from zero (when waiting for the other vessel to reach its setpoint) to a non-zero constant value (when both vessels can move again).
%This filtering of the reference signal is also depicted in Fig. ~\ref{fig:Lead-Controller}.

Next, we explain how the revised normalized-feedback variant is derived, and show that, when coupled with the plants in Eq.~\eqref{eq:surge_plant} and Eq.~\eqref{eq:yaw_plant}, the closed-loop system is stable, while improving the the transient response. 
%as explained in App. ~\ref{app:NormalizedFeedbackTF}.
%
% \subsection{Normalized feedback analysis in terms of the closed-loop transfer functions (surge and yaw)}
% \label{app:NormalizedFeedbackTF}
% Our implementation of the lead controller (App. ~\ref{app:feedbackImplementation}) preserves the
% closed-loop stability proved earlier, while improving the transient response. 

To make this more explicit, we compare the complementary sensitivity function (CSF) for two architectures: (i) the standard series lead controller with unity feedback, and (ii) the normalized lead placed in the feedback branch, with the static gain kept in the main branch (Fig.~\ref{fig:Lead-Controller}). The CSF maps the reference input $\mathrm{Ref}(t)$ to the closed-loop output $y(t)$.
In the Laplace domain, with $Y(s)=\mathcal{L}\{y(t)\}$ and $\mathrm{Ref}(s)=\mathcal{L}\{\mathrm{Ref}(t)\}$,
the closed-loop transfer function is $T(s)=\frac{Y(s)}{\mathrm{Ref}(s)}$. 

The CSF $T(s)$ provides direct insight into the closed-loop \emph{static tracking gain} ($T(0)$), which indicates the ratio between a steady-state output and a constant reference input, and the effective \emph{bandwidth} (how fast the system can track changes in the reference)~\cite{ModernControl}. By analyzing its structure for both architectures, we can show that the stability guarantees (Claim~\ref{clm:uy_exp}) extend to the normalized setting.

For each architecture (i) and (ii), we first establish stability, which is a prerequisite for discussing static gain and transient response.
We keep the notation from the stability analysis.
For surge we use $(\beta_u,\Omega_{c_u},\gamma_u)$ with
$K_{p_u}=\frac{\Omega_{c_u}}{\gamma_u}$.
For yaw we use $(\beta_\omega,\Omega_{c_\omega},\gamma_\omega)$
with $K_{p_\omega}=\frac{\Omega_{c_\omega}^2}{\gamma_\omega}$.
In the normalized topology, the main-branch gain is
$\frac{K_p}{\sqrt{\beta}}$, and the normalized lead  \[
H(s) = \frac{\beta s + \sqrt{\beta} \Omega_c}{s + \sqrt{\beta} \Omega_c}
\]
is placed in the feedback path. Next, we consider the individual surge and yaw dynamics (respectively). 

\niceparagraph{Surge dynamics (single integrator plant).}
Consider the surge plant
\[
P_u(s) = \frac{\gamma_u}{s}.
\]
%
%\textbf{1) Standard series lead.}
With the standard series lead
\[
C_u(s)=K_{p_u}\,
\frac{\sqrt{\beta_u}s+\Omega_{c_u}}
{s+\sqrt{\beta_u}\Omega_{c_u}},
\]
the complementary sensitivity becomes
\[
T^{\mathrm{std}}_u(s)
=
\frac{\Omega_{c_u}(\sqrt{\beta_u}s+\Omega_{c_u})}
{\bm{s^2+2\sqrt{\beta_u}\Omega_{c_u}s+\Omega_{c_u}^2}}.
\]
%\textbf{2) Normalized lead in feedback.}
With the normalized lead in feedback
\[
C^{\mathrm{main}}_u(s)=\frac{K_{p_u}}{\sqrt{\beta_u}},
\qquad
H_u(s)=
\frac{\beta_u s+\sqrt{\beta_u}\Omega_{c_u}}
{s+\sqrt{\beta_u}\Omega_{c_u}},
\]
we obtain
\[
T^{\mathrm{norm}}_u(s)
=
\frac{\tfrac{\Omega_{c_u}}{\sqrt{\beta_u}}s+\Omega_{c_u}^2}
{\bm{s^2+2\sqrt{\beta_u}\Omega_{c_u}s+\Omega_{c_u}^2}}.
\]

%\textbf{Interpretation.}
Both complementary sensitivity configurations share the  characteristic
polynomial
\[
s^2+2\sqrt{\beta_u}\Omega_{c_u}s+\Omega_{c_u}^2,
\]
hence the same stability condition (Claim~\ref{clm:uy_exp}).
Moreover,
\[
T^{\mathrm{std}}_u(0)
=
T^{\mathrm{norm}}_u(0)
=
1,
\]
so constant references are tracked without steady-state error.
The only structural difference is the zero location:
\[
z_u^{\mathrm{std}}
=
-\frac{\Omega_{c_u}}{\sqrt{\beta_u}},
\qquad
z_u^{\mathrm{norm}}
=
-\sqrt{\beta_u}\Omega_{c_u}.
\]
Since $\beta_u>1$, the normalized zero lies further left in the
complex plane. Consequently, its influence on the dominant
closed-loop poles is reduced, yielding smaller overshoot and
less transient peaking.

\niceparagraph{Yaw dynamics (double integrator plant).}
Consider the yaw plant
\[
P_\omega(s) = \frac{\gamma_\omega}{s^2}.
\]
%
%\textbf{1) Standard series lead.}
With the standard series lead
\[
C_\omega(s)=K_{p_\omega}
\frac{\sqrt{\beta_\omega}s+\Omega_{c_\omega}}
{s+\sqrt{\beta_\omega}\Omega_{c_\omega}},
\]
the complementary sensitivity is
\[
T^{\mathrm{std}}_\omega(s)
=
\frac{\Omega_{c_\omega}^2(\sqrt{\beta_\omega}s+\Omega_{c_\omega})}
{\bm{s^3+\sqrt{\beta_\omega}\Omega_{c_\omega}s^2
+\sqrt{\beta_\omega}\Omega_{c_\omega}^2 s
+\Omega_{c_\omega}^3}}.
\]
%
%\textbf{2) Normalized lead in feedback.}
With the normalized lead in feedback
\[
C^{\mathrm{main}}_\omega(s)
=
\frac{K_{p_\omega}}{\sqrt{\beta_\omega}},
\qquad
H_\omega(s)
=
\frac{\beta_\omega s+\sqrt{\beta_\omega}\Omega_{c_\omega}}
{s+\sqrt{\beta_\omega}\Omega_{c_\omega}},
\]
we obtain
\[
T^{\mathrm{norm}}_\omega(s)
=
\frac{\tfrac{\Omega_{c_\omega}^2}{\sqrt{\beta_\omega}}s
+\Omega_{c_\omega}^3}
{\bm{s^3+\sqrt{\beta_\omega}\Omega_{c_\omega}s^2
+\sqrt{\beta_\omega}\Omega_{c_\omega}^2 s
+\Omega_{c_\omega}^3}}.
\]

%\textbf{Interpretation.}
Here, again, both realizations share the same characteristic
denominator and therefore the same stability condition
(Claim~\ref{clm:uy_exp}), as well as unity DC gain:
\[
T_\omega(0)=1.
\]
The difference lies exclusively in the zero:
\[
z_\omega^{\mathrm{std}}
=
-\frac{\Omega_{c_\omega}}{\sqrt{\beta_\omega}},
\qquad
z_\omega^{\mathrm{norm}}
=
-\sqrt{\beta_\omega}\Omega_{c_\omega}.
\]
As in the surge case, the normalized implementation pushes the
zero further left. This attenuates overshoot and reduces
high-frequency amplification without altering the stability
proofs of Sec.~\ref{Surge-Yaw-Proof}.

This proves that the normalized feedback realization preserves the closed-loop
characteristic polynomial and unity steady-state gain, while
repositioning the controller zero deeper in the left half-plane.
The result is improved transient behavior (reduced overshoot
and peaking) without modifying the stability guarantees.    
}

\end{document}